\documentclass{article}
\usepackage{amsmath,amssymb}
\usepackage{amsthm,bm,framed,url}
\usepackage[ruled,lined,linesnumbered]{algorithm2e}
\usepackage{graphicx}
\usepackage[monochrome]{color}	
\usepackage{xcolor}
\usepackage{psfrag}
\usepackage{epstopdf}
\usepackage{wrapfig}
\usepackage{caption,subcaption}
\usepackage{cite}
\usepackage{fullpage}

\newtheorem{proposition}{Proposition}

\providecommand{\keywords}[1]{\textbf{Keywords:} #1}

\DeclareMathOperator*{\minimize}{minimize}

\newcommand{\st}{\text{subject to}}
\newcommand{\half}{\frac{1}{2}}

\DeclareMathOperator{\tr}{trace}

\DeclareMathOperator*{\OD}{\odot}
\newcommand{\kr}[1]{\OD_{\substack{j=1\\j \neq #1}}^N}
\DeclareMathOperator*{\CD}{\circledast}
\newcommand{\hada}[1]{\CD_{\substack{j=1\\j \neq #1}}^N}

\newcommand{\Y}{\mathbf{Y}}
\newcommand{\W}{\mathbf{W}}
\renewcommand{\H}{\mathbf{H}}
\newcommand{\A}{\mathbf{A}}
\newcommand{\B}{\mathbf{B}}

\newcommand{\I}{\mathbf{I}}
\newcommand{\G}{\mathbf{G}}
\newcommand{\F}{\mathbf{F}}

\newcommand{\U}{\mathbf{U}}
\newcommand{\V}{\mathbf{V}}
\renewcommand{\L}{\mathbf{L}}
\newcommand{\N}{\mathbf{N}}
\newcommand{\T}{\mathbf{T}}
\newcommand{\D}{\mathbf{D}}
\newcommand{\Sb}{\mathbf{S}}

\newcommand{\x}{\mathbf{x}}

\newcommand{\z}{\mathbf{z}}

\renewcommand{\u}{\mathbf{u}}

\renewcommand{\c}{\mathbf{c}}
\renewcommand{\d}{\mathbf{d}}

\newcommand{\R}{\mathbb{R}}
\renewcommand{\O}{\mathcal{O}}

\newcommand{\Yt}{\underline{\Y}}
\newcommand{\Vt}{\underline{\V}}

\newcommand{\Yd}[1]{\Y_{\!(#1)\!}}

\begin{document}

\title{A Flexible and Efficient Algorithmic Framework for Constrained Matrix and Tensor Factorization}

\author{Kejun Huang,
Nicholas D. Sidiropoulos,
and
Athanasios P. Liavas
}

\maketitle

\begin{abstract}
We propose a general algorithmic framework for constrained matrix and tensor factorization, which is widely used in signal processing and machine learning. The new framework is a hybrid between alternating optimization (AO) and the alternating direction method of multipliers (ADMM): each matrix factor is updated in turn, using ADMM, hence the name AO-ADMM. This combination can naturally accommodate a great variety of constraints on the factor matrices, and almost all possible loss measures for the fitting. Computation caching and warm start strategies are used to ensure that each update is evaluated efficiently, while the outer AO framework exploits recent developments in block coordinate descent (BCD)-type methods which help ensure that every limit point is a stationary point, as well as faster and more robust convergence in practice. Three special cases are studied in detail: non-negative matrix/tensor factorization, constrained matrix/tensor completion, and dictionary learning. Extensive simulations and experiments with real data are used to showcase the effectiveness and broad applicability of the proposed framework.
\end{abstract}

\section{Introduction}
Constrained matrix and tensor factorization techniques are widely used for latent parameter estimation and blind source separation in signal processing, dimensionality reduction and clustering in machine learning, and numerous other applications in diverse disciplines, such as chemistry and psychology. Least-squares low-rank factorization of matrices and tensors without additional constraints is relatively well-studied, as in the matrix case the basis of any solution is simply the principal components of the singular value decomposition (SVD)~\cite{eckart1936approximation}, also known as principal component analysis (PCA), and in the tensor case alternating least squares (ALS) and other algorithms usually yield satisfactory results~\cite{tomasi2006comparison}. ALS is also used for matrix factorization, especially when the size is so large that performing the exact PCA is too expensive.

Whereas unconstrained matrix and tensor factorization algorithms are relatively mature, their {\em constrained} counterparts leave much to be desired as of this writing, and a unified framework that can easily and naturally incorporate multiple constraints on the latent factors is sorely missing. Existing algorithms are usually only able to handle one or at most few specialized constraints, and/or the algorithm needs to be redesigned carefully if new constraints are added. Commonly adopted constraints imposed on the latent factors include non-negativity \cite{lee1999learning}, sparsity (usually via sparsity-inducing $l_1$ regularization \cite{olshausen1997sparse}), and simplex constraints~\cite{hofmann1999probabilistic}, to name just a few.

On top of the need to incorporate constraints on the latent factors, many established and emerging signal processing applications entail cost ({\em loss}) functions that differ from classical least-squares. Important examples include matrix completion \cite{candes2009exact} where missing values are ignored by the loss function, and robust PCA \cite{candes2011robust} where the $l_1$ loss is used. In the matrix case without constraints on the latent factors, these can be formulated as convex problems and solved in polynomial-time. With explicit constraints imposed on the latent factors, and/or for tensor data, however, non-convex (multi-linear) formulations are unavoidable, and a unified algorithmic framework that can handle a variety of constraints and loss functions would be very welcome.

In this paper, we propose a general algorithmic framework that seamlessly and relatively effortlessly incorporates many common types of constraints and loss functions, building upon and bridging together the alternating optimization (AO) framework and the alternating direction method of multipliers (ADMM), hence the name AO-ADMM.

While combining these frameworks may seem conceptually straightforward at first sight, what is significant is that AO-ADMM outperforms all prior algorithms for constrained matrix and tensor factorization under nonparametric constraints on the latent factors. One example is non-negative matrix factorization, where the prior art includes decades of research. This is the biggest but not the only advantage of AO-ADMM. Carefully developing various aspects of this combination, we show that
\begin{itemize}
\item AO-ADMM converges to a stationary point of the original NP-hard problem;
\item Using computation caching, warm-start, and good parameter settings, its per-iteration complexity is similar to that of ALS;
\item AO-ADMM can incorporate a wide-range of constraints and regularization penalties on the latent factors at essentially the same complexity;
\item It can also accommodate a wide variety of cost / loss functions, with only moderate increase in complexity relative to the classical least-squares loss; and
\item The core computations are exactly the same as ALS for unconstrained factorization, with some additional element-wise operations to handle constraints, making it easy to incorporate smart implementations of ALS, including sparse, parallel, and high-performance computing enhancements.  
\end{itemize}


\subsection{Notation}
We denote the (approximate) factorization of a matrix $\Y\approx\W\H^T$, where $\Y$ is $m \times n$, $\W$ is $m \times k$, and $\H$ is $n \times k$, with $k \leq m, n$, and in most cases much smaller. Note that adding constraints on $\W$ and $\H$ may turn the solution from easy to find (via SVD) but non-identifiable, to NP-hard to find but identifiable. It has been shown that simple constraints like non-negativity and sparsity can make the factors essentially unique, but at the same time, computing the optimal solution becomes NP-hard---see \cite{huang2014tsp} and references therein.

An $N$-way array of dimension $n_1 \times n_2 \times ... \times n_N$, with $N \geq 3$, is denoted with an underscore, e.g., $\Yt$. In what follows, we focus on the so-called parallel factor analysis (PARAFAC) model, also known as canonical decomposition (CANDECOMP) or canonical polyadic decomposition (CPD), which is essentially unique under mild conditions \cite{sidiropoulos2000uniqueness}, but constraints certainly help enhance estimation performance, and even identifiability. The factorization is denoted as $\Yt \approx [\H_d]_{d=1}^N$, which is a concise way of representing the model
\[
\Yt(i_1,...,i_N) \approx \sum_{j=1}^{k} \prod_{d=1}^{N} \H_d(i_d,j),~~
\forall i_1,...,i_N.
\]
Each matrix $\H_d$ is of size $n_d \times k$, corresponding to the factor of the $d$-th mode.

\subsection{Multi-linear Algebra Basics}
With the increasing interest in tensor data processing, there exist many tutorials on this topic, for example, \cite{smilde2005multi,kolda2009tensor}. Here we briefly review some basic multi-linear operations that will be useful for the purposes of this paper, and refer the readers to those tutorials and the references therein for a more comprehensive introduction.

The {\bf mode-$d$ matricization}, also known as mode-$d$ matrix unfolding, of $\Yt$, denoted as $\Yd{d}$, is a matrix of size $\prod_{j=1,j \neq d}^{N}n_j \times n_d$. Each row of $\Yd{d}$ is a vector obtained by fixing all the indices of $\Yt$ except the $d$-th one, and the matrix is formed by stacking these row vectors by traversing the rest of the indices from $N$ back to 1. As an example, for $N=3$, the three matricizations are
\[
\Yd{\!1\!} \!=\!\! \begin{bmatrix}
\Yt(:,1,:)^T \\ \vdots \\ \Yt(:,\!n_2\!,:)^T
\end{bmatrix}\!,
\Yd{\!2\!} \!=\!\! \begin{bmatrix}
\Yt(1,:,:)^T \\ \vdots \\ \Yt(\!n_1\!,:,:)^T
\end{bmatrix}\!,
\Yd{\!3\!} \!=\!\! \begin{bmatrix}
\Yt(1,:,:)   \\ \vdots \\ \Yt(\!n_1\!,:,:)
\end{bmatrix}\!.
\]
Notice that, though essentially in the same spirit, this definition of mode-$d$ matricization may be different from other expressions that have appeared in the literature, but we adopt this one for ease of our use.

The {\bf Khatri-Rao product} of matrices $\A$ and $\B$ having the same number of columns, denoted as $\A \odot \B$, is defined as the column-wise Kronecker product of $\A$ and $\B$. More explicitly, if $\A$ is of size $n \times k$, then
\[
\begin{aligned}
\A \odot \B =& \begin{bmatrix}
\A(:,1) \otimes \B(:,1) & \cdots & \A(:,k) \otimes \B(:,k)
\end{bmatrix} \\
=&\begin{bmatrix}
\A(1,1)\B(:,1) & & \A(1,k)\B(:,k) \\
\vdots & \cdots & \vdots \\
\A(n,1)\B(:,1) & & \A(n,k)\B(:,k)
\end{bmatrix}
\end{aligned}
\]
{\color{red}The Khatri-Rao product is associative (although not commutative).}
We {\color{red}therefore} generalize the operator $\odot$ to accept more than two arguments in the following way
\[
\kr{d}\H_i = \H_1 \odot \cdots \odot \H_{d-1} \odot \H_{d+1} \odot \cdots \odot \H_N.
\]
With the help of this notation, if $\Yt$ admits an exact PARAFAC model $\Yt = [\H_d]_{d=1}^N$, then it can be expressed in matricized form as
\[
\Yd{d} = \left( \kr{d}\H_j \right) \H_d^T.
\]

Lastly, a nice property of the Khatri-Rao product is that
\[
\left(\A \odot \B\right)^T\left(\A \odot \B\right) =
\A^T\A \circledast \B^T\B,
\]
where $\circledast$ denotes the element-wise (Hadamard) matrix product. More generally, it holds that
\[
\left( \kr{d}\H_j \right)^T\left( \kr{d}\H_j \right) = \hada{d}\H_j^T\H_j.
\]

\section{Alternating Optimization Framework: Preliminaries}
We start by formulating the factorization problem as an optimization problem in the most general form
\begin{equation}\label{problem:factorization}
\minimize_{\H_1,...,\H_N}~~
			l\left( \Yt - [\H_d]_{d=1}^N \right)
			+ \sum_{d=1}^{N} r_d(\H_d),
\end{equation}
with a slight abuse of notation by assuming $N$ can also take the value of 2. In (\ref{problem:factorization}), $l(\cdot)$ can be any loss measure, most likely separable down to the entries of the argument, and $r_d(\H_d)$ is the generalized regularization on $\H_d$, which may take the value of $+\infty$ so that any hard constraints can also be incorporated. For example, if we require that the elements of $\H_d$ are nonnegative, denoted as $\H_d \geq 0$, then
\[
r_d(\H_d) = \left\{
\begin{array}{ll}
0,	& \H_d \geq 0, \\
+\infty, & \text{otherwise.}
\end{array}
\right.
\]

Because of the multi-linear term $[\H_d]_{d=1}^N$, the regularized fitting problem is non-convex, and in many cases NP-hard \cite{vavasis2009complexity,hillar2013most}. A common way to handle this is to use the alternating optimization (AO) technique, i.e., update each factor $\H_d$ in a cyclic fashion. The popular ALS algorithm is a special case of this when $l(\cdot)$ is the least-squares loss, and there is no regularization. In this section, we will first revisit the ALS algorithm, with the focus on the per-iteration complexity analysis. Then, we will briefly discuss the convergence of the AO framework, especially some recent advances on the convergence of the traditional block coordinate descent (BCD) algorithm.

\subsection{Alternating Least-Squares Revisited}
Consider the unconstrained matrix factorization problem
\begin{equation}\label{eq:ls}
\minimize_{\W,\H}~~ \half\| \Y - \W\H^T \|_F^2,
\end{equation}
and momentarily ignore the fact that the optimal solution of (\ref{eq:ls}) is given by the SVD. The problem (\ref{eq:ls}) is non-convex in $\W$ and $\H$ jointly, but is convex if we fix one and only treat the other as variable. Supposing $\W$ is fixed, the sub-problem for $\H$ becomes the classical linear least squares and, if $\W$ has full column rank, the unique solution is given by
\begin{equation}\label{eq:lssol}
\H^T = (\W^T\W)^{-1} \W^T\Y.
\end{equation}
In practice, the matrix inverse $(\W^T\W)^{-1}$ is almost never explicitly calculated. Instead, the Cholesky decomposition of the Gram matrix $\W^T\W$ is computed, and for each column of $\W^T\Y$, a forward and a backward substitution are performed to get the corresponding column of $\H^T$. Since $\W$ is $m \times k$ and $\Y$ is $m \times n$, forming $\W^T\W$ and $\W^T\Y$ takes $\O(mk^2)$ and $\O(mnk)$ flops, respectively, computing the Cholesky decomposition requires $\O(k^3)$ flops, and finally the back substitution step takes $\O(nk^2)$ flops, similar to a matrix multiplication. If $m,n > k $, then the overall complexity is $\O(mnk)$.

An important implication is the following. Clearly, if $n=1$, then the cost of a least squares is $\O(mk^2)$. However, as $n$ grows, the complexity does not simply grow proportionally with $n$, but rather goes to $\O(mnk)$. The reason is that, although it seems we are now trying to solve $n$ least-squares problems, they all share the same matrix $\W$, thus the Cholesky decomposition of $\W^T\W$ can be reused throughout. This is a very nice property of the unconstrained least squares problems, which can be exploited to improve the computational efficiency of the ALS algorithm.

One may recall that another well-adopted method for least-squares is to compute the QR decomposition of $\W$ as $\W={\bf QR}$, so that $\H^T = {\bf R}^{-1}{\bf Q}^T\Y$.  This can be shown to give the same computational complexity as the Cholesky version, and is actually more stable numerically. However, if $\W$ has some special structure, it is easier to exploit this structure if we use Cholesky decomposition. Therefore, in this paper we only consider solving least-squares problems using the Cholesky decomposition.

One important structure that we encounter is in the tensor case. For the ALS algorithm for PARAFAC, the update of $\H_d$ is the solution of the following least squares problem
\[
\minimize_{\H_d}~~ \half\left\| \Yd{d} -
				   \left( \kr{d}\H_j \right)\H_d^T
				   	\right\|_F^2,
\]
and the solution is given by
\[
\H_d^T = \left( \hada{d}\H_j^T\H_j \right)^{-1}
		\left( \kr{d}\H_j \right)^T \Yd{d}.
\]
As we can see, the Gram matrix is computed efficiently by exploiting the structure, and its Cholesky decomposition can be reused. The most expensive operation is actually the computation of $( \OD_{j \neq d} \H_j )^T \Y_{(d)}$, but very efficient algorithms for this (that work without explicitly forming the Khatri-Rao product and the $d$-mode matricization) are available \cite{bader2007efficient,kang2012gigatensor,niranjay,choi2014dfacto,smith2015splatt}. If we were to adopt the QR decomposition approach, however, none of these methods could be applied.


In summary, least squares is a very mature technique with many favorable properties that render the ALS algorithm very efficient. On the other hand, most of the algorithms that deal with problems with constraints on the factors or different loss measures do not inherit these good properties. The goal of this paper is to propose an AO-based algorithmic framework, which can easily handle many types of constraints on the latent factors and many loss functions, with per-iteration complexity essentially the same as the complexity of an ALS step.

\subsection{The Convergence of AO}
Consider the following (usually non-convex) optimization problem with variables separated into $N$ blocks, each with its own constraint set
\begin{equation}\label{eq:bcd}
\begin{aligned}
\minimize_{\x_1,...,\x_N}~~ & f(\x_1,...,\x_N) \\
\st~~ & \x_d \in {\cal X}_d,~~ \forall d=1,...,N.
\end{aligned}
\end{equation}
A classical AO method called block coordinate descent (BCD) cyclically updates $\x_d$ via solving
\begin{equation}\label{eq:bcd_update}
\begin{aligned}
\minimize_{\bf \xi}~~ & f(\x_1^{r+1},...,\x_{d-1}^{r+1},{\bf \xi},\x_{d+1}^r,...,\x_N^r) \\
\st~~ & {\bf \xi} \in {\cal X}_d,
\end{aligned}
\end{equation}
at the ($r+1$)-st iteration \cite[\S\,2.7]{bertsekas1999nonlinear}. Obviously, this will decrease the objective function monotonically. {\color{blue} If some additional assumptions are satisfied, then we can have stronger convergence claims \cite[Proposition 2.7.1]{bertsekas1999nonlinear}. Simply put, if the sub-problem (\ref{eq:bcd_update}) is convex and has a \emph{unique} solution, then every limit point is a stationary point; furthermore, if ${\cal X}_1,...,{\cal X}_N$ are all compact, {\color{red} which implies that the sequence generated by BCD is bounded,} then BCD is guaranteed to converge to a stationary point, even if (\ref{eq:bcd}) is non-convex~\cite{tseng2001convergence}.}

In many cases (\ref{eq:bcd_update}) is convex, but the uniqueness of the solution is very hard to guarantee. A special case that does not require uniqueness, first noticed by Grippo and Sciandrone \cite{Grippo2000}, is when $N=2$. On hindsight, this can be explained by the fact that for $N=2$, BCD coincides with the so-called maximum block improvement (MBI) algorithm \cite{chen2012maximum}, which converges under very mild conditions. However, instead of updating the blocks cyclically, MBI only updates the one block that decreases the objective the most, thus the per-iteration complexity is $(N-1)$ times higher than BCD; therefore MBI is not commonly used in practice when $N$ is large.

Another way to ensure convergence, proposed by Razaviyayn {\it et al.} \cite{razaviyayn2013unified}, is as follows. Instead of updating $\x_d$ as the solution of (\ref{eq:bcd_update}), the update is obtained by solving a majorized version of (\ref{eq:bcd_update}), called the block successive upper-bound minimization (BSUM). The convergence of BSUM is essentially the same, but now we can deliberately design the majorizing function to ensure that the solution is unique. One simple way to do this is to put a proximal regularization term
\begin{equation}\label{eq:bsum_update}
\begin{aligned}
\minimize_{\bf \xi}~ & f(\x_1^{r\!+\!1}\!,...,\x_{d\!-\!1}^{r\!+\!1}\!,{\bf \xi},\x_{d\!+\!1}^r\!,...,\!\x_N^r)
						\!+\! \frac{\mu}{2}\|{\bf \xi} \!-\! \x_d^r\|^2 \\
\st~ & ~~{\bf \xi} \in {\cal X}_d,
\end{aligned}
\end{equation}
for some $\mu > 0$, where $\x_d^r$ is the update of $\x_d$ from the previous iteration. If (\ref{eq:bcd_update}) is convex, then (\ref{eq:bsum_update}) is strongly convex, which gives a unique minimizer.
{\color{blue} Thus, the algorithm is guaranteed to converge to a stationary point, as long as the sequence generated by the algorithm is bounded. Similar results are also proved in \cite{xu2013block}, where the authors used a different majorization for constrained matrix/tensor factorization; we shall compare with them in numerical experiments.}

\section{Solving the Sub-problems Using ADMM}
The AO algorithm framework is usually adopted when each of the sub-problems can be solved efficiently. This is indeed the case for the ALS algorithm, since each update is in closed-form. For the general factorization problem (\ref{problem:factorization}), we denote the sub-problem as
\begin{equation}\label{eq:loss}
\minimize_{\H}~~ l(\Y-\W\H^T) + r(\H).
\end{equation}
For the matrix case, this is simply the sub-problem for the right factor, and one can easily figure out the update of the left factor by transposing everything; for the tensor case, this becomes the update of $\H_d$ by setting $\Y=\Yd{d}$ and $\W=\OD_{j \neq d} \H_j$. This is for ease of notation, as these matricizations and Khatri-Rao products need not be actually computed explicitly. Also notice that this is the sub-problem for the BCD algorithm, and for better convergence we may want to add a proximal regularization term to (\ref{eq:loss}), which is very easy to handle, thus omitted here.

We propose to use the alternating direction method of multipliers (ADMM) to solve (\ref{eq:loss}). ADMM, if used in the right way, inherits a lot of the good properties that appeared in each update of the ALS method. Furthermore, the AO framework naturally provides good initializations for ADMM, which further accelerates its convergence for the subproblem. As a preview, the implicit message here is that {\em closed-form solution is not necessary for computational efficiency}, as we will explain later. After a brief introduction of ADMM, we first apply it to (\ref{eq:loss}) which has least-squares loss, and then generalize it to universal loss measures.

\subsection{Alternating Direction Method of Multipliers}
ADMM solves convex optimization problems that can be written in the form
\[
\begin{aligned}
\minimize_{\x,\z}~~ & f(\x) + g(\z) \\
\st~~ & \A\x + \B\z = \c,
\end{aligned}
\]
by iterating the following updates
\begin{equation*}
\begin{aligned}
\x & \leftarrow \arg\min_{\x} f(\x) +
				(\rho/2)\| \A\x + \B\z - \c + \u\|_2^2, \\
\z & \leftarrow \arg\min_{\z} g(\z) +
				(\rho/2)\| \A\x + \B\z - \c + \u \|_2^2, \\
\u & \leftarrow \u + (\A\x+\B\z-\c),
\end{aligned}
\end{equation*}
where $\u$ is a scaled version of the dual variables corresponding to the equality constraint $\A\x+\B\z=\c$, and $\rho$ is specified by the user.

A comprehensive review of the ADMM algorithm can be found in \cite{Boyd2011} and the references therein. The beauty of ADMM is that it converges under mild conditions (in the convex case), while artful splitting of the variables into the two blocks $\x$ and $\z$ can yield very efficient updates, and/or distributed implementation. Furthermore, if $f$ is strongly convex and Lipschitz continuous, then linear convergence of ADMM can be achieved; cf. guidelines on the optimal step-size $\rho$ in \cite[\S~9.3]{monotone}, and \cite{ghadimi2015optimal} for an analysis of ADMM applied to quadratic programming.

\subsection{Least-Squares Loss}
We start by considering $l(\cdot)$ in (\ref{eq:loss}) to be the least-squares loss $(1/2)\|\cdot\|_F^2$. The problem is reformulated by introducing a $k \times n$ auxiliary variable $\tilde{\H}$
\begin{equation}\label{problem:least-squares}
\begin{aligned}
\minimize_{\H,\tilde{\H}}~~ & \half \| \Y - \W\tilde{\H} \|_F^2 + r(\H)\\
\st~~ & \H = \tilde{\H}^T.
\end{aligned}
\end{equation}

It is easy to adopt the ADMM algorithm and derive the following iterates:
\begin{equation}\label{ADMM:least-squares}
\begin{aligned}
\tilde{\H} & \leftarrow (\W^T\W + \rho\I)^{-1}
							(\W^T\Y + \rho(\H+\U)^T),\\
\H & \leftarrow \arg\min_{\H} r(\H) +
				\frac{\rho}{2} \| \H - \tilde{\H}^T + \U \|_F^2, \\
\U & \leftarrow \U + \H - \tilde{\H}^T.
\end{aligned}
\end{equation}
One important observation is that, throughout the iterations the same matrix $\W^T\Y$ and matrix inversion $(\W^T\W + \rho\I)^{-1}$ are used. Therefore, to save computations, we can cache $\W^T\Y$ and the Cholesky decomposition of $\W^T\W + \rho\I=\L\L^T$. Then the update of $\tilde{\H}$ is dominated by one forward substitution and one backward substitution, resulting in a complexity of $\O(k^2n)$.

The update of $\H$ is the so-called \emph{proximity operator} of the function $(1/\rho)r(\cdot)$ around point $(\tilde{\H}^T-\U)$, and in particular if $r(\cdot)$ is the indicator function of a convex set, then the update of $\H$ becomes a projection operator, a special case of the proximity operator. For a lot of regularizations/constraints, especially those that are often used in matrix/tensor factorization problems, the update of $\H$ boils down to element-wise updates, i.e., costing $\O(kn)$ flops. Here we list some of the most commonly used constraints/regularizations in the matrix factorization problem, and refer the reader to \cite[\S 6]{Parikh2014}. For simplicity of notation, let us define $\bar{\H}=\tilde{\H}^T-\U$.
\begin{itemize}
\item Non-negativity. In this case $r(\cdot)$ is the indicator function of $\R_+$, and the update is simply zeroing out the negative values of $\bar{\H}$. In fact, any element-wise bound constraints can be handled similarly, since element-wise projection is trivial.
\item Lasso regularization. For $r(\H)=\lambda\|\H\|_1$, the sparsity inducing regularization, the update is the well-known \emph{soft-thresholding} operator: $h_{ij}=[1-(\lambda/\rho)|\bar{h}_{ij}|^{-1}]_+\bar{h}_{ij}$. The element-wise thresholding can also be converted to block-wise thresholding if one wants to impose structured sparsity, leading to the group Lasso regularization.
\item Simplex constraint. In some probabilistic model analysis we need to constrain the columns or rows to be element-wise non-negative and sum up to one. As described in \cite{duchi2008efficient}, this projection can be done with a randomized algorithm with linear-time complexity on average.
\item Smoothness regularization. We can encourage the columns of $\H$ to be smooth by adding the regularization $r(\H)=(\lambda/2)\|\T\H\|_F^2$ where $\T$ is an $n \times n$ tri-diagonal matrix with 2 on the diagonal and $-1$ on the super- and sub-diagonal. Its proximity operator is given by $\H = \rho(\lambda\T^T\T+\rho\I)^{-1}\bar{\H}$. Although it involves a large matrix inversion, notice that it has a fixed bandwidth of 2, thus can be efficiently calculated in $\O(kn)$ time \cite[\S 4.3]{golub2012matrix}.
\end{itemize}

We found empirically that by setting $\rho=\|\W\|_F^2/k$, the ADMM iterates for the regularized least-squares problem (\ref{problem:least-squares}) converge very fast. This choise of $\rho$ can be seen as an approximation to the optimal $\rho$ given in \cite{monotone}, but much cheaper to obtain. With a good initialization, naturally provided by the AO framework, the update of $\H$ usually does not take more than $5$ or $10$ ADMM iterations, and very soon reduces down to only 1 iteration. The proposed algorithm for the sub-problem (\ref{problem:least-squares}) is summarized in Alg. \ref{alg:ADMMls}. As we can see, the pre-calculation step takes $\O(k^2m + k^3)$ flops to form the Cholesky decomposition, and $\O(mnk)$ flops to form $\F$. Notice that these are actually the only computations in Alg.~\ref{alg:ADMMls} that involve $\W$ and $\Y$, which implies that in the tensor case, all the tricks to compute $\W^T\W$ and $\W^T\Y$ can be applied here, and then we do not need to worry about them anymore. The computational load of each ADMM iteration is dominated by the $\tilde{\H}$-update, with complexity $\O(k^2n)$.

It is interesting to compare Alg.~\ref{alg:ADMMls} with an update of the ALS algorithm, whose complexity is essentially the same as the pre-calculation step plus one iteration. For a small number of ADMM iterations, the complexity of Alg. \ref{alg:ADMMls} is of the same order as an ALS step.

\begin{algorithm}[h]
\KwIn{$\Y$, $\W$, $\H$, $\U$, $k$}
Initialize $\H$ and $\U$\;
$\G = \W^T\W$\;
$\rho = \tr(\G)/k$ \;
Calculate $\L$ from the Cholesky decomposition of $\G+\rho\I=\L\L^T$\;
$\F=\W^T\Y$ \;
\Repeat{$r<\varepsilon$ and $s<\varepsilon$}{
$\tilde{\H} \leftarrow \L^{-T}\L^{-1}(\F+\rho(\H+\U)^T)$
		using forward/backward substitution \;
$\H \leftarrow \arg\min_{\H} r(\H) +
				\frac{\rho}{2} \| \H - \tilde{\H}^T + \U \|_F^2$ \;
$\U \leftarrow \U + \H - \tilde{\H}^T$ \;
}(\hfill $r$ and $s$ defined in {(\ref{eq:r})} and {(\ref{eq:s})})
\Return $\H$ and $\U$.
\caption{Solve (\ref{problem:least-squares}) using ADMM}
\label{alg:ADMMls}
\end{algorithm}

For declaring termination, we adopted the general termination criterion described in \cite[\S 3.3.1]{Boyd2011}. After some calibration, we define the relative primal residual
\begin{equation}\label{eq:r}
r = \| \H-\tilde{\H}^T \|_F^2 / \| \H \|_F^2,
\end{equation}
and the relative dual residual
\begin{equation}\label{eq:s}
s = \| \H - \H_0 \|_F^2 / \| \U \|_F^2,
\end{equation}
where $\H_0$ is $\H$ from the previous ADMM iteration, and terminate Alg.~\ref{alg:ADMMls} if both of them are smaller than some threshold.

Furthermore, if the BSUM framework is adopted, we need to solve a proximal regularized version of (\ref{problem:least-squares}), and that term can easily be absorbed into the update of $\tilde{\H}$.

\subsection{General Loss}\label{sec:loss}
Now let us derive an ADMM algorithm to solve the more general problem (\ref{eq:loss}). For this case, we reformulate the problem by introducing two auxiliary variables $\tilde{\H}$ and $\tilde{\Y}$
\begin{equation}\label{problem:general_loss}
\begin{aligned}
\minimize_{\H, \tilde{\H}, \tilde{\Y}}~~& l(\Y-\tilde{\Y}) + r(\H) \\
\st~~ & \H = \tilde{\H}^T,~~\tilde{\Y} = \W\tilde{\H}.
\end{aligned}
\end{equation}
To apply ADMM, let $\tilde{\H}$ be the first block, and $(\tilde{\Y},\H)$ be the second block, and notice that in the second block update $\tilde{\Y}$ and $\H$ can in fact be updated independently. This yields the following iterates:
\begin{equation}\label{ADMM:general_loss}
\begin{aligned}
&\,~~~\tilde{\H} \leftarrow (\W^{\!T}\W + \rho\I)^{\!-\!1}
		(\W^{\!T\!}(\tilde{\Y} \!+\! \V) +
		\rho(\H \!+\! \U)^{\!T})\\
&\left\{
\begin{aligned}
\H &\leftarrow \arg\min_{\H}~ r(\H) +
					\frac{\rho}{2} \| \H - \tilde{\H}^T + \U \|_F^2, \\
\tilde{\Y} &\leftarrow \arg\min_{\tilde{\Y}}~ l(\Y-\tilde{\Y}) +
					\half \|\tilde{\Y} - \W\tilde{\H} + \V\|_F^2,
\end{aligned}
\right.	\\
&\left\{
\begin{aligned}
\U &\leftarrow \U + \H - \tilde{\H}^T, \\
\V &\leftarrow \V + \tilde{\Y} - \W\tilde{\H}.
\end{aligned}
\right.
\end{aligned}
\end{equation}
where $\U$ is the scaled dual variable corresponding to the constraint $\H = \tilde{\H}^T$, and $\V$ is the scaled dual variable corresponding to the equality constraint $\tilde{\Y}=\W\tilde{\H}$. Notice that we set the penalty parameter $\rho$ corresponding to the second constraint to be 1, since it works very well in practice, and also leads to very intuitive results for some loss functions. This can also be interpreted as first pre-conditioning this constraint to be $\frac{1}{\sqrt{\rho}}\tilde{\Y} = \frac{1}{\sqrt{\rho}}\W\tilde{\H}$, and then a common $\rho$ is used. Again we set $\rho=\|\W\|_F^2/k$.

As we can see, the update of $\tilde{\H}$ is simply a linear least squares problem, and all the previous discussion about caching the Cholesky decomposition applies. It is also easy to absorb an additional proximal regularization term into the update of $\tilde{\H}$, if the BSUM framework is adopted. The update of $\tilde{\Y}$ is (similar to the update of $\H$) a proximity operator, and since almost all loss functions we use are element-wise, the update of $\tilde{\Y}$ is also very easy. The updates for some of the most commonly used non-least-squares loss functions are listed below. For simplicity, we define $\bar{\Y} = \W\tilde{\H} - \V$, similar to the previous sub-section.

\begin{itemize}
\item Missing values. In the case that only a subset of the entries in $\Y$ are available, a common way to handle this is to simply fit the low-rank model only to the available entries. Let ${\cal A}$ denote the index set of the available values in $\Y$, then the loss function becomes $l(\Y-\tilde{\Y})= \half\sum_{(i,j)\in \cal A}(y_{ij}-\tilde{y}_{ij})^2$. Thus, the update of $\tilde{\Y}$ in (\ref{ADMM:general_loss}) becomes
\[
\tilde{y}_{ij} = \left\{\begin{array}{ll}
\half ( y_{ij} + \bar{y}_{ij} ), & ~~(i,j) \in {\cal A}, \\
\bar{y}_{ij}, & ~~\text{otherwise}.
\end{array}
\right.
\]
\item Robust fitting. In the case that data entries are not uniformly corrupted by noise but only sparingly corrupted by outliers, or when the noise is dense but heavy-tailed (e.g., Laplacian-distributed), we can use the $l_1$ norm as the loss function for robust (resp. maximum-likelihood) fitting, i.e., $l(\Y-\tilde{\Y})=\|\Y-\tilde{\Y}\|_1$. This is similar to the $l_1$ regularization, and the element-wise update is
\[
\tilde{y}_{ij} = \left\{\begin{array}{ll}
y_{ij},~&~|\bar{y}_{ij}-y_{ij}| \leq 1, \\
\bar{y}_{ij}-1,~&~\bar{y}_{ij}-y_{ij} > 1, \\
\bar{y}_{ij}+1,~&~\bar{y}_{ij}-y_{ij} < -1.
\end{array}
\right.
\]

\item Huber fitting. Another way to deal with possible outliers in $\Y$ is to use the Huber function to measure the loss $l(\Y-\tilde{\Y})=\sum_{i,j}\phi_{\lambda}(y_{ij}-\tilde{y}_{ij})$ where
\[
\phi_{\lambda}(z) = \left\{\begin{array}{ll}
\half z^2,~&~|z| \leq \lambda, \\
\lambda|z|-\half\lambda^2,~&~ \text{otherwise.}
\end{array}
\right.
\]
The element-wise closed-form update is
\[
\tilde{y}_{ij}=\left\{\begin{array}{ll}
\half(\bar{y}_{ij}+y_{ij}),~&~|\bar{y}_{ij}-y_{ij}| \leq 2\lambda,\\
\bar{y}_{ij}-\lambda,~&~ \bar{y}_{ij}-y_{ij} > 2\lambda,\\
\bar{y}_{ij}+\lambda,~&~ \bar{y}_{ij}-y_{ij} < -2\lambda.
\end{array}
\right.
\]

\item Kullback-Leibler divergence. A commonly adopted loss function for non-negative integer data is the Kullback-Leibler (K-L) divergence defined as
\[
D(\Y||\tilde{\Y}) = \sum_{i,j} \left(
y_{ij}\log\frac{y_{ij}}{\tilde{y}_{ij}} - y_{ij} + \tilde{y}_{ij}
\right)
\]
for which the proximity operator is
\[
\tilde{\Y} = \half \left(
(\bar{\Y}-1) + \sqrt{(\bar{\Y}-1)^2+4\Y}
\right) ,
\]
where all the operations are taken element-wise \cite{sun2014alternating}. Furthermore, the K-L divergence is a special case of certain families of divergence functions, such as $\alpha$-divergence and $\beta$-divergence \cite{cichocki2009fast}, whose corresponding updates are also very easy to derive (boil down to the proximity operator of a scalar function).
\end{itemize}

An interesting observation is that if the loss function is in fact the least-squares loss, the matrix $(\tilde{\Y}+\V)$ that $\tilde{\H}$ is trying to fit in (\ref{ADMM:general_loss}) is the data matrix $\Y$ {\it per se}. Therefore, the update rule (\ref{ADMM:general_loss}) boils down to the update rule (\ref{ADMM:least-squares}) in the least-squares loss case, with some redundant updates of $\tilde{\Y}$ and $\V$. The detailed ADMM algorithm for (\ref{problem:general_loss}) is summarized in Alg.~\ref{alg:ADMMgl}. We use the same termination criterion as in Alg.~\ref{alg:ADMMls}.

\begin{algorithm}[h]
\KwIn{$\Y$, $\W$, $\H$, $\U$, $\tilde{\Y}$, $\V$, $k$}
Initialize $\H$, $\U$, $\tilde{\Y}$, and $\V$\;
$\G = \W^T\W$\;
$\rho = \tr(\G)/k$ \;
Calculate $\L$ from the Cholesky decomposition of $\G+\rho\I=\L\L^T$\;
\Repeat{$r<\varepsilon$ and $s<\varepsilon$}{
$\tilde{\H} \leftarrow \L^{-T}\L^{-1}(\W^{T}(\tilde{\Y} + \V)+\rho(\H+\U)^T)$
		using forward/backward substitution \;
$\H \leftarrow \arg\min_{\H} r(\H) +
				\frac{\rho}{2} \| \H - \tilde{\H}^T + \U \|_F^2$ \;
$\tilde{\Y} \leftarrow \arg\min_{\tilde{\Y}}~ l(\Y-\tilde{\Y}) +
					\half \|\tilde{\Y} - \W\tilde{\H} + \V\|_F^2$ \;
$\U \leftarrow \U + \H - \tilde{\H}^T$ \;
$\V \leftarrow \V + \tilde{\Y} - \W\tilde{\H}$ \;
}(\hfill $r$ and $s$ defined in {(\ref{eq:r})} and {(\ref{eq:s})})
\Return $\H$, $\U$, $\tilde{\Y}$, and $\V$.
\caption{Solve (\ref{problem:general_loss}) using ADMM}
\label{alg:ADMMgl}
\end{algorithm}

Everything seems to be in place to seamlessly move from the least-squares loss to arbitrary loss. Nevertheless, closer scrutiny reveals that some compromises must be made to take this leap. One relatively minor downside is that with a general loss function we may lose the linear convergence rate of ADMM -- albeit with the good initialization naturally provided by the AO framework and our particular choice of $\rho$, it still converges very fast in practice. The biggest drawback is that, by introducing the auxiliary variable $\tilde{\Y}$ and its dual variable $\V$, the big matrix product $\W^T(\tilde{\Y}+\V)$ must be re-computed in each ADMM iteration, whereas in the previous case one only needs to compute $\W^T\Y$ once. This is the price we must pay; but it can be moderated by controlling the maximum number of ADMM iterations.

\noindent {\bf Scalability considerations:}
As big data analytics become increasingly common, it is important to keep scalability issues in mind as we develop new analysis methodologies and algorithms. Big data $\Yt$ is usually stored as a sparse array, i.e., a list of (\texttt{index,value}) pairs, with the unlisted entries regarded as zeros or missing.
With the introduction of $\tilde{\Y}$ and $\V$, both of size($\Yt$), one hopes to be able to avoid dense operations. Fortunately, for some commonly used loss functions, this is possible. Notice that by defining $\bar{\Y}=\W\tilde{\H}-\V$, the $\V$-update essentially becomes
\[
\V \leftarrow \tilde{\Y} - \bar{\Y},
\]
which means a significant portion of entries in $\V$ are constants---0 if the entries are regarded as missing, $\pm 1$ or $\pm \lambda$ in the robust fitting or Huber fitting case if the entries are regarded as ``corrupted''---thus they can be efficiently stored as a sparse array. As for $\tilde{\Y}$, one can simply generate it ``on-the-fly'' using the closed-form we provided earlier (notice that $\bar{\Y}$ has the memory-efficient ``low-rank plus sparse'' structure). The only occasion that $\tilde{\Y}$ is needed is when computing $\W^T\tilde{\Y}$.

\section{Summary of the Proposed Algorithm}
We propose to use Alg.~\ref{alg:ADMMls} or \ref{alg:ADMMgl} as the core sub-routine for alternating optimization. The proposed ``universal'' multi-linear factorization algorithm is summarized as Alg.~\ref{alg:AO-ADMM}. A few remarks on implementing Alg.~\ref{alg:AO-ADMM} are in order.

\begin{algorithm}[t]
Initialize $\H_1$,...,$\H_N$\;
Initialize $\U_1$,...,$\U_N$ to be all zero matrices\;
\uIf{least-squares loss}{
\Repeat{convergence}{
\For{$d=1,...,N$}{
$\Y=\Yd{d}$ and $\W=\OD_{j \neq d} \H_j$
\tcp*{\footnotesize not necessarily formed explicitly}
update $\H_d$ and $\U_d$ using Alg.~\ref{alg:ADMMls} initialized with the previous $\H_d$ and $\U_d$\;
}
update $\mu$ if necessary
\tcp*{\footnotesize refer to (\ref{eq:mu})}}
}
\Else{
Initialize $\tilde{\Yt}\leftarrow \Yt$, $\Vt \leftarrow \underline{\bf 0}$\;
\Repeat{convergence}{
\For{$d=1,...,N$}{
$\Y=\Yd{d}$ and $\W=\OD_{j \neq d} \H_j$
\tcp*{\footnotesize not necessarily formed explicitly}
update $\H_d$, $\!\!\U_d$, $\!\!\tilde{\Y}_{\!(d)\!}$, $\!\!\V_{\!(d)\!}$ using Alg.~\ref{alg:ADMMgl} initialized with the previous $\H_d$, $\!\!\U_d$, $\!\!\tilde{\Y}_{\!(d)\!}$, $\!\!\V_{\!(d)\!}$\;
}
update $\mu$ if necessary
\tcp*{\footnotesize refer to (\ref{eq:mu})}}
}
\caption{Proposed algorithm for (\ref{problem:factorization})}
\label{alg:AO-ADMM}
\end{algorithm}

Since each factor $\H_d$ is updated in a cyclic fashion, one expects that after a certain number of cycles $\H_d$ (and its dual variable $\U_d$) obtained in the previous iteration will not be very far away from the update for the current iteration. In this sense, the outer AO framework naturally provides a good initial point to the inner ADMM iteration. With this warm-start strategy, the optimality gap for the sub-problem is then bounded by the per-step improvement of the AO algorithm, which is small. This mode of operation is crucial for insuring the efficiency of Alg.~\ref{alg:AO-ADMM}. Our experiments suggest that soon after an initial transient stage, the sub-problems can be solved in just one ADMM iteration (with reasonable precision).

Similar ideas can be used for $\tilde{\Y}$ and $\V$ in the matrix case if we want to deal with non-least-squares loss, and actually only one copy of them is needed in the updates of both factors. A few different options are available in the tensor case. If memory is not an issue in terms of the size of $\Yt$, a convenient approach that is commonly adopted in ALS implementations is to store all $N$ matricizations $\Yd{1},...,\Yd{N}$, so they are readily available without need for repetitive data re-shuffling during run-time. If this practice is adopted, then it makes sense to also have $N$ copies of $\tilde{\Yt}$ and $\underline{\V}$, in order to save computation. Depending on the size and nature of the data and how it is stored, it may be completely unrealistic to keep multiple copies of the data and the auxiliary variables, at which point our earlier discussion on scalable implementation of Alg.~\ref{alg:ADMMgl} for big but sparse data can be instrumental.

Sometimes an additional proximal regularization is added to the sub-problems. The benefit is two-fold: it helps the convergence of the AO outer-loop when $N \geq 3$; while for the ADMM inner-loop it improves the conditioning of the sub-problem, which may accelerate the convergence of ADMM, especially in the general loss function case when we do not have strong convexity. {\color{blue} The convergence of AO-ADMM is summarized in Proposition~\ref{prop:convergence}.
\begin{proposition}\label{prop:convergence}
If the sequence generated by AO-ADMM in Alg.~\ref{alg:AO-ADMM} is bounded, then for
\begin{enumerate}
\item $N=2$,
\item $N>2, \mu>0$,
\end{enumerate}
AO-ADMM converges to a stationary point of~(\ref{problem:factorization}).
\end{proposition}
\begin{proof}
The first case with $\mu=0$ is covered in \cite[Theorem~3.1]{chen2012maximum}, and
the cases when $\mu>0$ are covered in \cite[Theorem~2]{razaviyayn2013unified}.
\end{proof}
Note that for $N=2$, using $\mu=0$ yields faster convergence than $\mu>0$. For $N>2$, i.e., for tensor data, we can update $\mu$ as follows
\begin{equation}\label{eq:mu}
\mu \leftarrow 10^{-7} + 0.01\frac{\|\Yt-[\H_d]_{d=1}^N\|}{\|\Yt\|},
\end{equation}
which was proposed in \cite{razaviyayn2013unified} for unconstrained tensor factorization, and works very well in our context as well.

The convergence result in Proposition~\ref{prop:convergence} has an additional assumption that the sequence generated by the algorithm is bounded. For unconstrained PARAFAC, diverging components may be encountered during AO iterations \cite{kruskal1989how,Stegeman:2014:FLD:2598946.2599268}, but adding Frobenious norm regularization for each matrix factor (with a small weight) ensures that the iterates remain bounded. }

As we can see, the ADMM is an appealing sub-routine for alternating optimization, leading to a simple plug-and-play generalization of the workhorse ALS algorithm. Theoretically, they share the same per-iteration complexity if the number of inner ADMM iterations is small, which is true in practice, after an initial transient. Efficient implementation of the overall algorithm should include data-structure-specific algorithms for $\W^T\Y$ or $(\OD_{j \neq d} \H_j)^T\Yd{d}$, which dominate the per-iteration complexity, and may include parallel/distributed computation along the lines of \cite{liavas2014parallel}.

Finally, if a non-least-squares loss is to be used, we suggest that the least-squares loss is first employed to get preliminary estimates (using Alg.~\ref{alg:AO-ADMM} calling Alg.~\ref{alg:ADMMls}) which can then be fed as initialization to run Alg.~\ref{alg:AO-ADMM} calling Alg.~\ref{alg:ADMMgl}. The main disadvantage of Alg.~\ref{alg:ADMMgl} compared to Alg.~\ref{alg:ADMMls} is that the big matrix (or tensor) multiplication $\W^T(\tilde{\Y}+\V)$ needs to be calculated in each ADMM iteration. Therefore, this strategy can save a significant amount of computations at the initial stage.

\section{Case Studies and Numerical Results}
In this section we will study some well-known constrained matrix/tensor factorization problems, derive the corresponding update for $\H$ in Alg.~\ref{alg:ADMMls} or $\H$ and $\tilde{\Y}$ in Alg.~\ref{alg:ADMMgl}, and compare it to some of the state-of-the-art algorithms for that problem. In all examples we denote our proposed algorithm as {\bf AO-ADMM}. All experiments are performed in MATLAB 2015a on a Linux server with 32 Xeon 2.00GHz cores and 128GB memory.

\subsection{Non-negative Matrix and Tensor Factorization}
Perhaps the most common constraint imposed on the latent factors is non-negativity -- which is often supported by physical considerations (e.g., when the latent factors represent chemical concentrations, or power spectral densities) or other prior information, or simply because non-negativity sometimes yields interpretable factors \cite{lee1999learning}. Due to the popularity and wide range of applications of NMF, numerous algorithms have been proposed for fitting the NMF model, and most of them can be easily generalized to the tensor case. After a brief review of the existing algorithms for NMF, we compare our proposed algorithm to some of the best algorithms reported in the literature to showcase the efficiency of AO-ADMM.

Let us start by considering NMF with least-squares loss, which is the prevailing loss function in practice. By adopting the alternating optimization framework, the sub-problem that emerges for each matrix factor is non-negative (linear)  least-squares (NNLS). Some of the traditional methods for NNLS are reviewed in \cite{chen2009nonnegativity} (interestingly, not including ADMM), and most of them have been applied to NMF or non-negative PARAFAC, e.g., the active-set (AS) method \cite{bro1997fast,kim2008nonnegative} and block-principle-pivoting (BPP) \cite{kim2011fast,kim2012fast}. Recall that in the context of the overall multi-linear factorization problem we actually need to solve a large number of (non-negative) least-squares problems sharing the same mixing matrix $\W$, and in the unconstrained case this means we only need to calculate the Cholesky factorization of $\W^T\W$ once. Unfortunately, this good property that enables high efficiency implementation of ALS is not preserved by either AS or BPP. Sophisticated methods that group similar rows to reduce the number of inversions have been proposed \cite{van2004fast}, although as $k$ grows larger this does not seem appealing in the worst case. Some other methods, like the multiplicative-update (MU) \cite{lee2001algorithms} or hierarchical alternating least squares (HALS) \cite{cichocki2009fast}, ensure that the per-iteration complexity is dominated by calculating $\W^T\W$ and $\W^T\Y$, although more outer-loops are needed for convergence. These are actually one step majorization-minimization or block coordinate descent applied to the NNLS problem. An accelerated version of MU and HALS is proposed in \cite{gillis2012accelerated}, which essentially does a few more inner-loops after computing the most expensive $\W^T\Y$.

ADMM, on the other hand, may not be the fastest algorithm for a single NNLS problem, yet its overhead can be amortized when there are many NNLS problem instances sharing the same mixing matrix, especially if good initialization is readily available. This is in contrast to an earlier attempt to adopt ADMM to NMF \cite{cai2013nonnegative}, which did not use Cholesky caching, warm start, and a good choice of $\rho$ to speed up the algorithm. Furthermore, ADMM can seamlessly incorporate different regularizations as well as non-least-squares loss.

We should emphasize that AO forms the backbone of our proposed algorithm -- ADMM is only applied to the sub-problems. There are also algorithms that directly apply an ADMM approach to the whole problem \cite{xu2012alternating,liavas2014parallel,sun2014alternating}. The per-iteration complexity of those algorithms is also the same as the unconstrained alternating least-squares. However, due to the non-convexity of the whole problem, the loss is not guaranteed to decrease monotonically, unlike alternating optimization. Moreover, both ADMM and AO guarantee that every limit point is a stationary point, but in practice AO almost always converges {\color{blue}(as long as the updates stay bounded)}, which is not the case for ADMM applied to the whole problem.

{\color{blue}
In another recent line of work \cite{xu2013block}, a similar idea of using an improved AO framework to ensure convergence is used. When \cite{xu2013block} is specialized to non-negative matrix/tensor factorization, each update becomes a simple proximal-gradient step with an extrapolation. The resulting algorithm is also guaranteed to converge (likewise assuming that the iterates remain bounded), but it turns out to be slower than our algorithm, as we will show in our experiments. 
}

To apply our proposed algorithm to NMF or non-negative PARAFAC with least-squares loss, Alg.~\ref{alg:ADMMls} is used to solve the sub-problems, with line~8 customized as
\[
\H \leftarrow \left[ \tilde{\H}^T - \U \right]_+,
\]
i.e., zeroing out the negative values of $(\tilde{\H}^T - \U)$. The tolerance for the ADMM inner-loop is set to $0.01$.

\subsubsection{Non-negative matrix factorization}
We compare AO-ADMM with the following algorithms:
\begin{itemize}
\item[] {\bf AO-BPP.} AO using block principle pivoting \cite{kim2011fast}\footnote{\url{http://www.cc.gatech.edu/~hpark/nmfsoftware.php}};
\item[] {\bf accHALS.} Accelerated HALS \cite{gillis2012accelerated}\footnote{\url{https://sites.google.com/site/nicolasgillis/code}};
\item[] {\color{blue} {\bf APG.} Alternating proximal gradient \cite{xu2013block}\footnote{\url{http://www.math.ucla.edu/~wotaoyin/papers/bcu/matlab.html}};}
\item[] {\bf ADMM.} ADMM applied to the whole problem \cite{xu2012alternating} \footnote{\url{http://mcnf.blogs.rice.edu/}}.
\end{itemize}
AO-BPP and HALS are reported in \cite{kim2011fast} to outperform other methods, accHALS is proposed in \cite{gillis2012accelerated} to improve HALS, APG is reported in \cite{xu2013block} to outperform AO-BPP, and we include ADMM applied to the whole problem to compare the convergence behavior of AO and ADMM for this non-convex factorization problem.

The aforementioned NMF algorithms are tested on two datasets. One is a dense image data set, the Extended Yale Face Database B\footnote{\url{http://vision.ucsd.edu/~leekc/ExtYaleDatabase/ExtYaleB.html}}, of size $32256 \times 1932$, where each column is a vectorized $168 \times 192$ image of a face, and the dataset is a collection of face images of 29 subjects under various poses and illumination conditions. The other one is the Topic Detection and Tracking 2 (TDT2) text corpus\footnote{\url{http://www.cad.zju.edu.cn/home/dengcai/Data/TextData.html}}, of size $10212 \times 36771$, which is a sparse document-term matrix where each entry counts the frequency of a term in one document.

The convergence of the relative error $\|\Y-\W\H^T\|_F/\|\Y\|_F$ versus time in seconds for the Extended Yale B dataset is shown in Fig. \ref{fig:YaleB}, with $k=100$ on the left and $k=300$ on the right; and for the TDT2 dataset in Fig. \ref{fig:TDT2}, with $k=500$ on the left and $k=800$ on the right. The ADMM algorithm \cite{xu2012alternating} is not included for TDT2 because the code provided online is geared towards imputation of matrices with missing values -- it does not treat a sparse input matrix as the full data, unless we fill-in all zeros.

\begin{figure}[t!]
\centering
\begin{subfigure}[t]{0.4\textwidth}
\includegraphics[width=\textwidth]{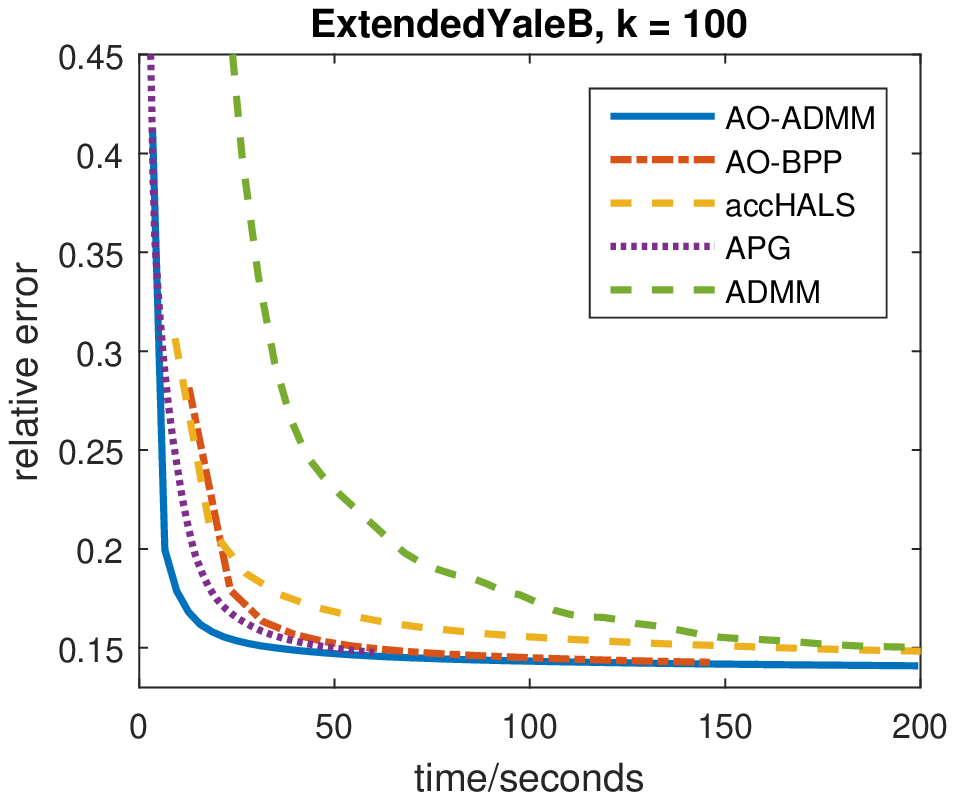}
\end{subfigure}
\begin{subfigure}[t]{0.4\textwidth}
\includegraphics[width=\textwidth]{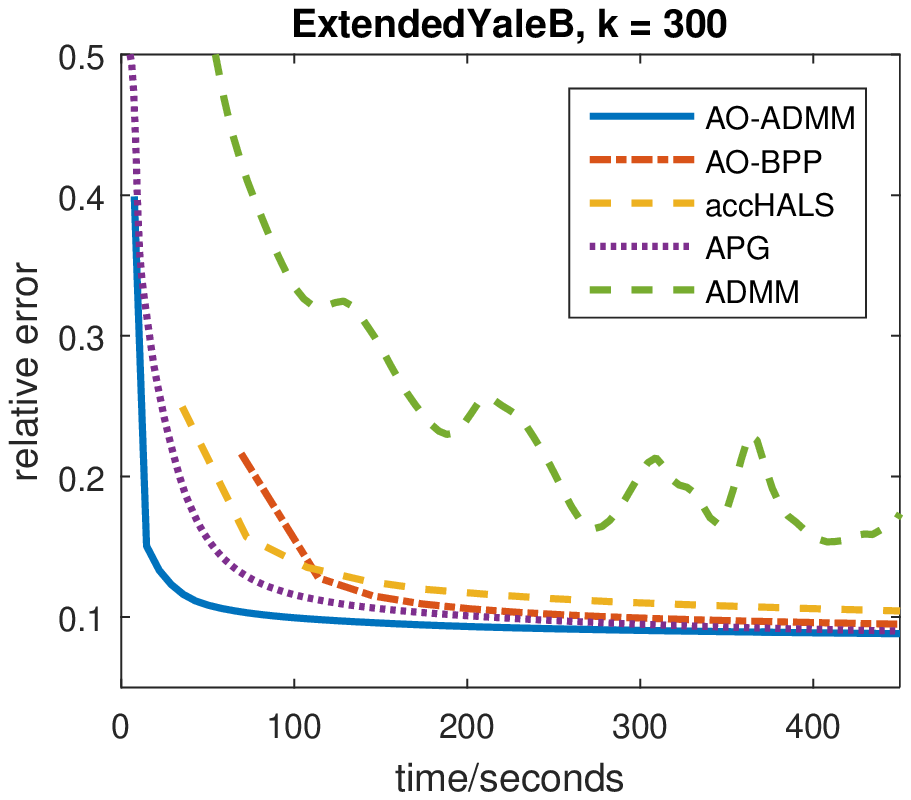}
\end{subfigure}
\caption{Convergence of some NMF algorithms on the Extended Yale B dataset.}
\label{fig:YaleB}
\end{figure}
\begin{figure}[t!]
\centering
\begin{subfigure}[t]{0.4\textwidth}
\includegraphics[width=\textwidth]{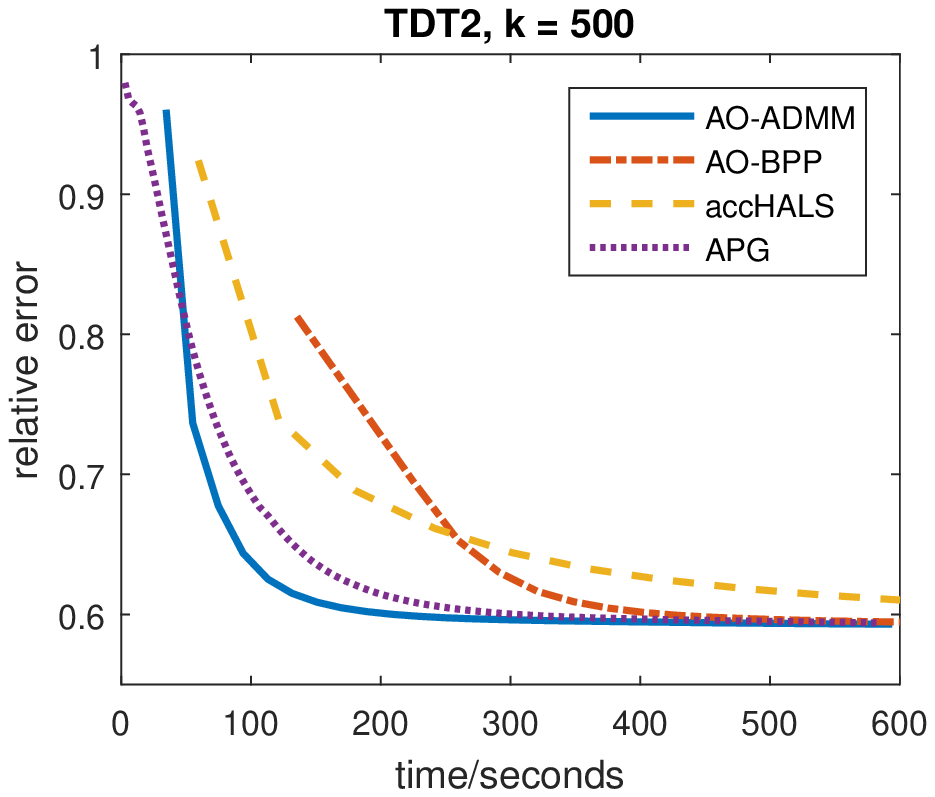}
\end{subfigure}
\begin{subfigure}[t]{0.4\textwidth}
\includegraphics[width=\textwidth]{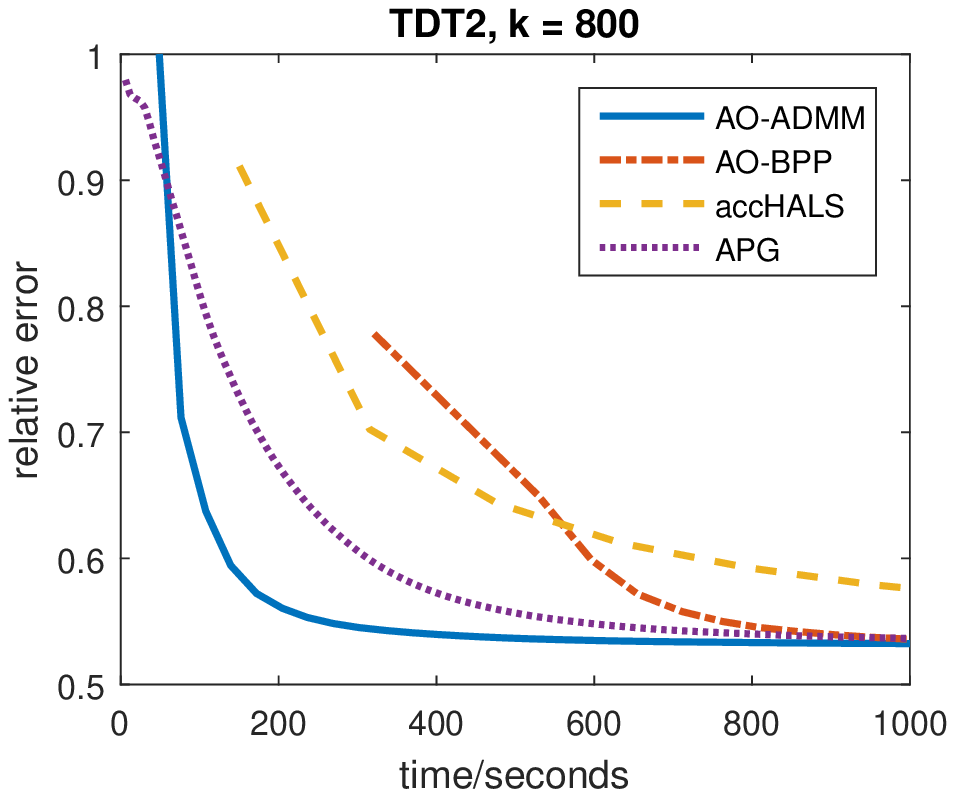}
\end{subfigure}
\caption{Convergence of some NMF algorithms on the TDT2 dataset.}
\label{fig:TDT2}
\end{figure}

\begin{table}[t!]
\caption{Averaged performance of NMF algorithms on synthetic data.}
\label{tab:NMF}
\small
\begin{center}
\renewcommand{\arraystretch}{1.2}
\begin{tabular}{lccc}
\hline
Algorithm    & $\|\Y-\W\H^T\|_F$ & run time & iterations \\
\hline
AO-ADMM & 193.1026 & 21.7s & 86.9 \\
AO-BPP 	& 193.1516 & 40.9s & 52.2 \\
accHALS	& 193.1389 & 26.8s & 187.0\\
APG		& 193.1431 & 25.3s & 240.2 \\
ADMM	& 193.6808 & 31.9s & 125.2\\
\hline
\end{tabular}
\end{center}
\vspace*{-0.3cm}
\end{table}
We also tested these algorithms on synthetic data. For $m=n=2000$ and $k=100$, the true $\W$ and $\H$ are generated by drawing their elements from an i.i.d. exponential distribution with mean 1, and then 50\% of the elements are randomly set to 0. The data matrix $\Y$ is then set to be $\Y=\W\H^T + \N$, where the elements of $\N$ are drawn from an i.i.d. Gaussian distribution with variance 0.01. The averaged results of 100 Monte-Carlo trials are shown in Table~\ref{tab:NMF}. As we can see, AO-based methods are able to attain smaller fitting errors than directly applying ADMM to this non-convex problem, while AO-ADMM provides the most efficient per-iteration complexity.

\subsubsection{Non-negative PARAFAC}
Similar algorithms are compared in the non-negative PARAFAC case:
\begin{itemize}
\item[] {\bf AO-BPP.} AO using block principle pivoting \cite{kim2012fast}\footnotemark[1];
\item[] {\bf HALS.} Hierarchical alternating least-squares \cite{cichocki2009fast}\footnotemark[1];
\item[] {\bf APG.} Alternating proximal gradient \cite{xu2013block}\footnotemark[2];
\item[] {\bf ADMM.} ADMM applied to the whole problem \cite{liavas2014parallel}.
\end{itemize}
For our proposed AO-ADMM algorithm, a diminishing proximal regularization term in the form (\ref{eq:bsum_update}) is added to each sub-problem to enhance the overall convergence, with the regularization parameter $\mu$ updated as (\ref{eq:mu}).

Two real datasets are being tested: one is a dense CT image dataset\footnote{\url{http://www.nlm.nih.gov/research/visible/}} of size $260 \times 190 \times 150$, which is a collection of 150 CT images of a female's ankle, each with size $260 \times 190$; the other one is a sparse social network dataset -- Facebook Wall Posts\footnote{\url{http://konect.uni-koblenz.de/networks/facebook-wosn-wall}}, of size $46952 \times 46951 \times 1592$, that collects the number of wall posts from one Facebook user to another over a period of 1592 days. The sparse tensor is stored in the \texttt{sptensor} format supported by the \texttt{tensor\_toolbox}\cite{TTB_Software}, and all the aforementioned algorithms use this toolbox to handle sparse tensor data.

\begin{figure}[t!]
\centering
\begin{subfigure}[t]{0.4\textwidth}
\includegraphics[width=\textwidth]{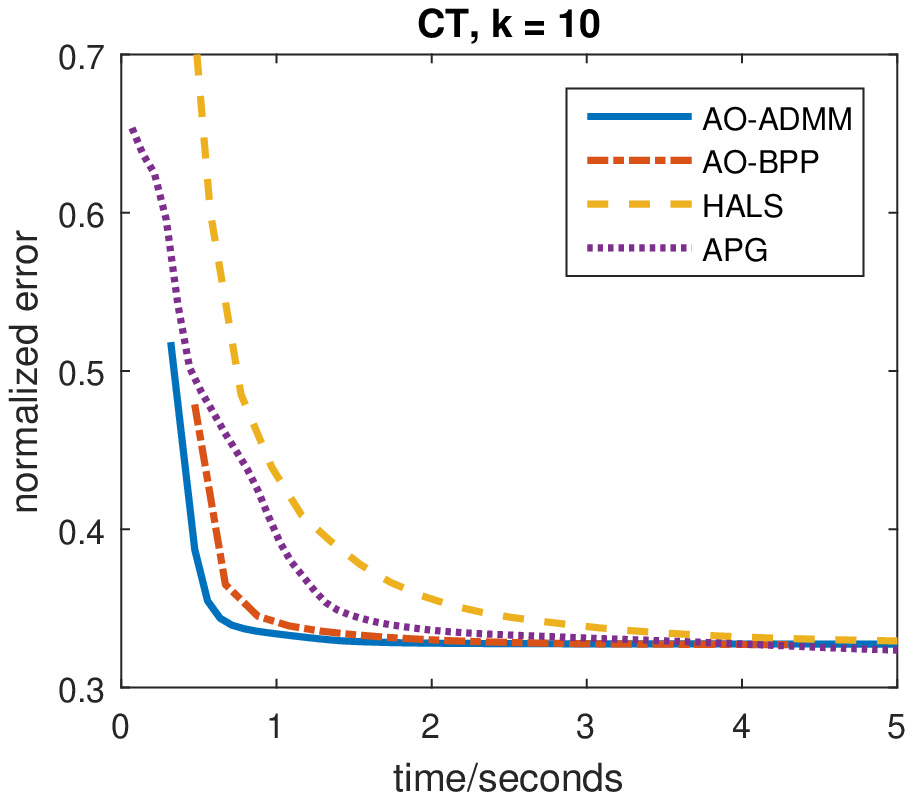}
\end{subfigure}
\begin{subfigure}[t]{0.4\textwidth}
\includegraphics[width=\textwidth]{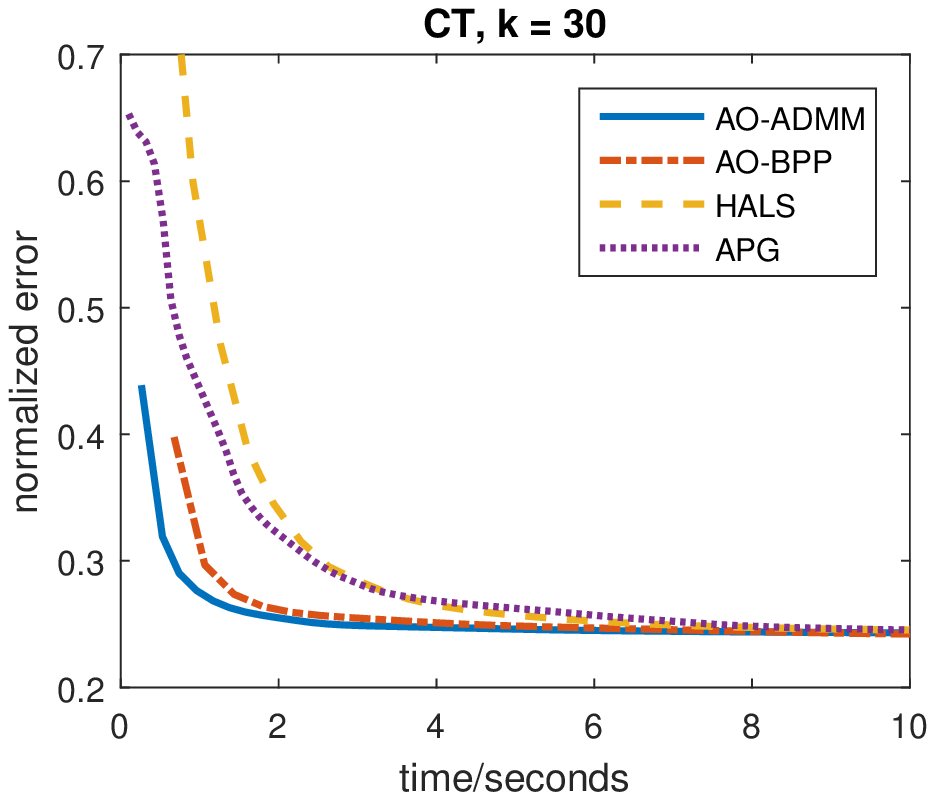}
\end{subfigure}
\caption{Convergence of some non-negative PARAFAC algorithms on the CT dataset.}
\label{fig:CT}
\end{figure}
\begin{figure}[t!]
\centering
\vspace{-1.3pt}
\begin{subfigure}[t]{0.4\textwidth}
\includegraphics[width=\textwidth]{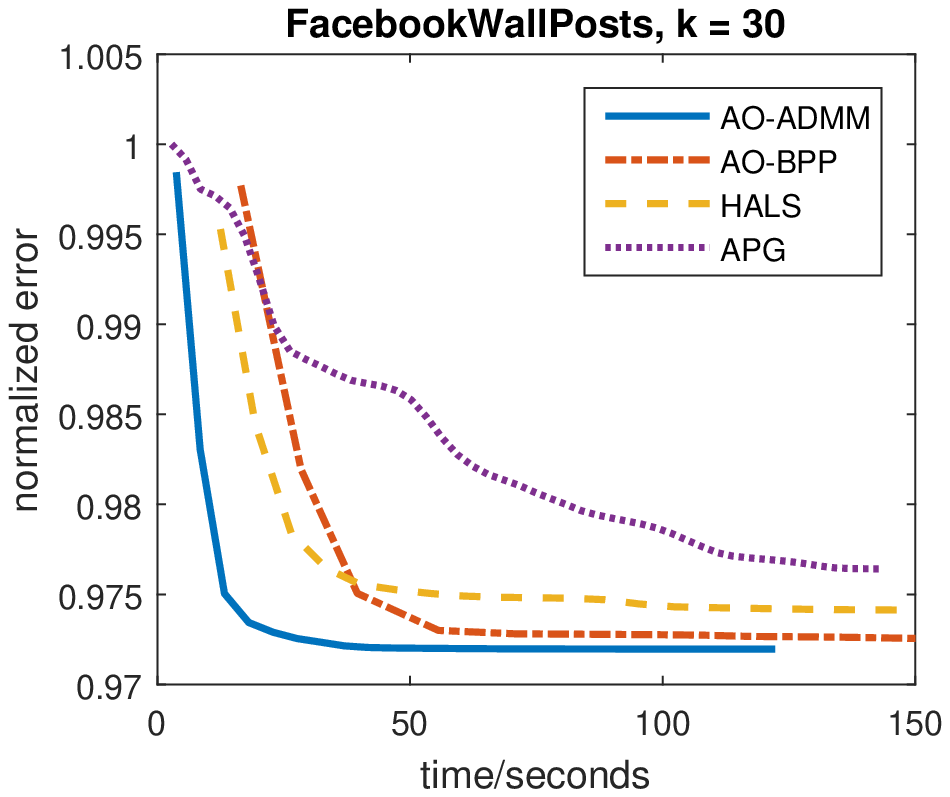}
\end{subfigure}
\begin{subfigure}[t]{0.4\textwidth}
\includegraphics[width=\textwidth]{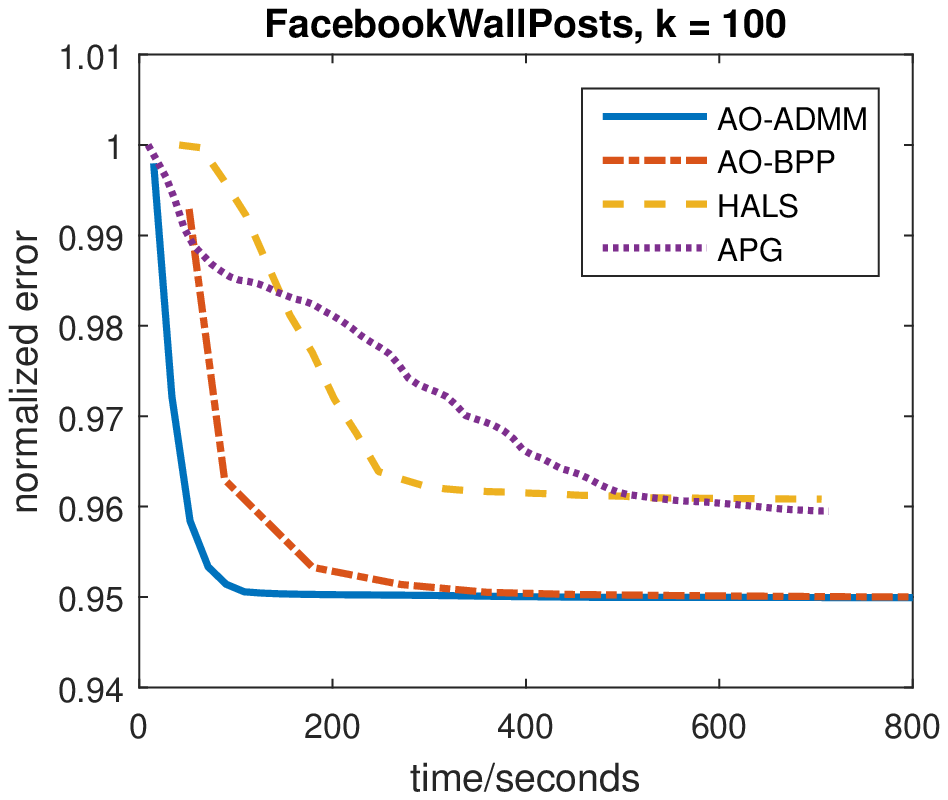}
\end{subfigure}
\caption{Convergence of some non-negative PARAFAC algorithms on the Facebook Wall Posts dataset.}
\label{fig:facebook}
\end{figure}

\begin{table}[t!]
\caption{Averaged performance of non-negative PARAFAC algorithms on synthetic data}
\label{tab:NTF}
\small
\begin{center}
\renewcommand{\arraystretch}{1.2}
\begin{tabular}{lccc}
\hline
Algorithm    & {\scriptsize $\|\Yt-[\H_1,\H_2,\H_3]\|$}
					& run time & iterations \\
\hline
AO-ADMM & 1117.597 & 145.2s & 25.1 \\
AO-BPP 	& 1117.728 & 679.0s & 22.6 \\
HALS	& 1117.655 & 1838.7s & 137.7\\
APG		& 1117.649 & 1077.4s & 156.3\\
ADMM	& 1156.799 & 435.9s & 77.2 \\
Tensorlab	& 1118.427 & 375.8s & N/A \\
\hline
\end{tabular}
\end{center}
\vspace*{-0.3cm}
\end{table}

Similar to the matrix case, the normalized root mean squared error versus time in seconds for the CT dataset is shown in Fig.~\ref{fig:CT}, with $k=10$ on the left and $k=30$ on the right, and that for the Facebook Wall Posts data is shown in Fig.~\ref{fig:facebook}, with $k=30$ on the left and $k=100$ on the right. As we can see, AO-ADMM again converges the fastest, not only because of the efficient per-iteration update from Alg.~\ref{alg:ADMMls}, but also thanks to the additional proximal regularization to help the algorithm avoid swamps, which are not uncommon in alternating optimization-based algorithms for tensor decomposition.

Monte-Carlo simulations were also conducted using synthetic data for 3-way non-negative tensors with $n_1=n_2=n_3=500$ and $k=100$, with the latent factors generated in the same manner as for the previous NMF synthetic data, and the tensor data generated as the low-rank model synthesized from those factors plus i.i.d. Gaussian noise with variance $0.01$. The averaged result over 100 trials is given in Table \ref{tab:NTF}. For this experiment we have also included \texttt{Tensorlab} \cite{tensorlab}, which handles non-negative PARAFAC using ``all-at-once'' updates based on the Gauss-Newton method. As we can see, AO-ADMM again outperforms all other algorithms in all cases considered.

\subsection{Constrained Matrix and Tensor Completion}
As discussed before, real-world data are often stored as a sparse array, i.e., in the form of (\texttt{index,value}) pairs. Depending on the application, the unlisted entries in the array can be treated as zeros, or as not (yet) observed but possibly nonzero. A well-known example of the latter case is the {\em Netflix prize problem}, which involves an array of movie ratings indexed by customer and movie. The data is extremely sparse, but the fact that a customer did not rate a movie does not mean that the customer's rating of that movie would be zero -- and the goal is actually to predict those unseen ratings to provide good movie recommendations.

For matrix data with no constraints on the latent factors, convex relaxation techniques that involve the matrix nuclear norm have been proposed with provable matrix reconstruction bounds \cite{candes2009exact}. Some attempts have been made to generalize the matrix nuclear norm to tensor data \cite{gandy2011tensor,liu2013tensor}, but that boils down to the Tucker model rather than the PARAFAC model that we consider here. A key difference is that Tucker modeling can only hope to impute (recover missing values) in the data, whereas PARAFAC can uniquely recover the latent factors -- the important `dimensions' of consumer preference in this context. Another key difference is that the aforementioned convex relaxation techniques cannot incorporate constraints on the latent factors, which can improve the estimation performance. Taking the Netflix problem as an example, {\em user-bias} and {\em movie-bias} terms are often successfully employed in recommender systems; these  can be easily subsumed in the factorization formulation by constraining, say, the first column of $\W$ and the second column of $\H$ to be equal to the all-one vector. Moreover, interpreting each column of $\W$ ($\H$) as the appeal of a certain movie genre to the different users (movie ratings for a given type of user, respectively), it is natural to constrain the entries of $\W$ and $\H$ to be non-negative.

When matrix/tensor completion is formulated as a constrained factorization problem using a loss function as in Sec. \ref{sec:loss}, there are traditionally two ways to handle it. One is directly using alternating optimization, although due to the random positions of the missing values, the least-squares problem for each row of $\H$ will involve a different subset of the rows of $\W$, thus making the update inefficient even in the unconstrained case. A more widely used way is an instance of expectation-maximization (EM): one starts by filling the missing values with zeros, and then iteratively fits a (constrained) low-rank model and imputes the originally missing values with predictions from the interim low-rank model. More recently, an ADMM approach that uses an auxiliary variable for the full data was proposed \cite{xu2012alternating}, although if we look carefully at that auxiliary variable, it is exactly equal to the filled-in data given by the EM method.

\begin{figure}[t!]
\centering
\includegraphics[width=.8\textwidth]{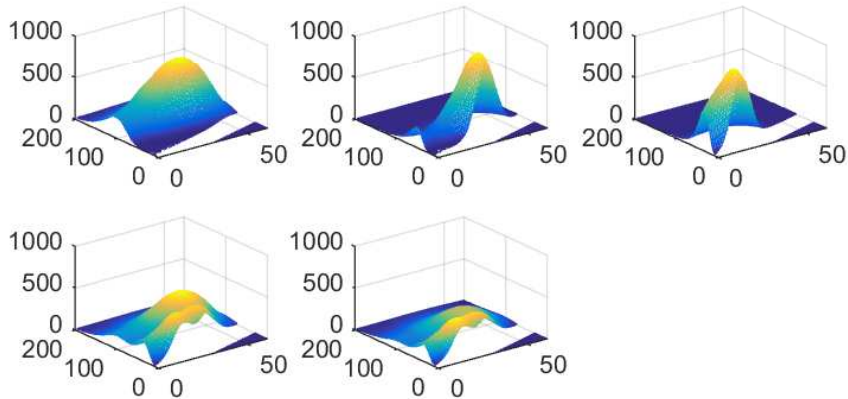}
\caption{Illustration of the missing values in the Amino acids fluorescence data.}
\label{fig:Ex2_samples}
\end{figure}
\begin{figure}[t!]
\centering
\includegraphics[width=.8\textwidth]{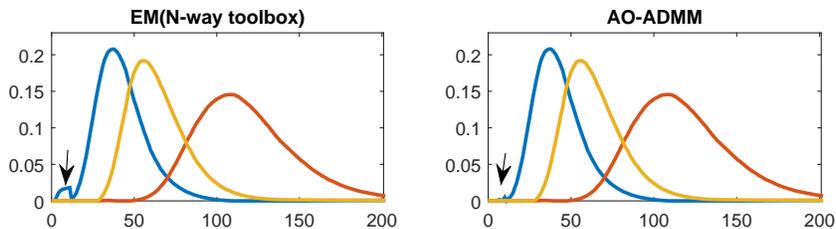}
\caption{The emission loadings ($\H_2$) produced by the $N$-way toolbox on the left, which uses EM, and by AO-ADMM on the right.}
\label{fig:Ex2_emissions}
\end{figure}
In fact, the auxiliary variable $\tilde{\Y}$ that we introduce is similar to that of \cite{xu2012alternating}, thus also related to the way that EM imputes the missing values---one can treat our method as imputing the missing values per ADMM inner-loop, the method in \cite{xu2012alternating} as imputing per iteration, and EM as imputing after several iterations. However, our proposed AO-ADMM is able to give better results than EM, despite the similarities. As an illustrative example, consider the Amino acids fluorescence data\footnote{\url{http://www.models.kvl.dk/Amino_Acid_fluo}}, which is a $5 \times 201 \times 61$ tensor known to be generated by a rank-3 non-negative PARAFAC model. However, some of the entries are known to be badly contaminated, and are thus deleted, as shown in Fig. \ref{fig:Ex2_samples}. Imposing non-negativity on the latent factors, the emission loadings $\H_2$ of the three chemical components provided by the EM method using the $N$-way toolbox\cite{andersson2000n} and AO-ADMM are shown in Fig. \ref{fig:Ex2_emissions}. While both results are satisfactory, AO-ADMM is able to suppress the artifacts caused by the systematically missing values in the original data, as indicated by the arrows in Fig. \ref{fig:Ex2_emissions}.

We now evaluate our proposed AO-ADMM on a movie rating dataset called MovieLens\footnote{\url{http://grouplens.org/datasets/movielens/}}, which consists of 100,000 movie ratings from 943 users on 1682 movies. MovieLens includes 5 sets of 80\%-20\% splits of the ratings for training and testing, and for each split we fit a matrix factorization model based on the 80\% training data, and evaluate the correctness of the model on the 20\% testing data. The averaged performance on this 5-fold cross validation is shown in Fig.~\ref{fig:movielens}, where we used the mean absolute error (MAE) for comparison with the classical collaborative filtering result \cite{sarwar2001item} (which attains a MAE of 0.73). On the left of Fig.~\ref{fig:movielens}, we used the traditional least-squares criterion to fit the available ratings, whereas on the right we used the Kullback-Leibler divergence for fitting, since it is a meaningful statistical model for integer data. For each fitting criterion, we compared the performance by imposing Tikhonov regularization $(\lambda/2)\|\cdot\|_F^2$ with $\lambda=0.1$, or non-negativity, or non-negativity with biases (i.e., in addition constraining the first column of $\W$ and second column of $\H$ to be all ones). Some observations are as follows:
\begin{itemize}
\item Low-rank indeed seems to be a good model for this movie rating data, and the right rank seems to be 4 or 5, higher rank leads to over-fitting, as evident from Fig.~\ref{fig:movielens};
\item Imposing non-negativity reduces the over-fitting at higher ranks, whereas the fitting criterion does not seem to be playing a very important role in terms of performance;
\item By adding biases, the best case prediction MAE at rank 4 is less than 0.69, an approximately 6\% improvement over the best result reported in \cite{sarwar2001item}.
\end{itemize}
Notice that our aim here is to showcase how AO-ADMM can be used to explore possible extentions to the matrix completion problem formulation, rather than come up with the best recommender system method, which would require significant exploration in its own right. We believe with the versability of AO-ADMM, researchers can easily test various models for matrix/tensor completion, and quickly narrow down the one that works the best for their specific application.

\begin{figure}[t!]
\centering
\includegraphics[width=.8\textwidth]{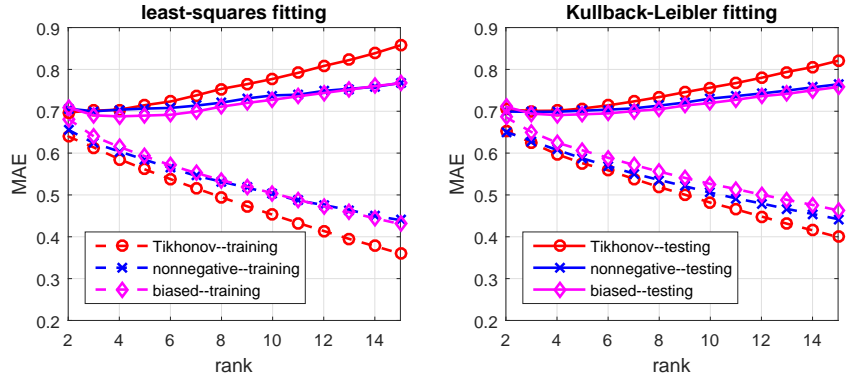}
\caption{Training and testing mean absolute error (MAE) versus model rank of the MovieLens data, averaged over a 5-fold cross validation, comparing least-squares fitting (on the left) and Kullback-Leibler fitting (on the right), with Tikhonov regularization, non-negativity constraint, or non-negativity with biases on the latent factors.}
\label{fig:movielens}
\end{figure}

\subsection{Dictionary Learning}
Many natural signals can be represented as an (approximately) sparse linear combination of some (possibly over-complete) basis, for example the Fourier basis for speech signals and the wavelet basis for images. If the basis (or {\em dictionary} when over-complete) is known, one can directly do data compression via greedy algorithms or convex relaxations to obtain the sparse representation\cite{bruckstein2009sparse}, or even design the sensing procedure to reduce the samples required for signal recovery\cite{candes2008introduction}. If the dictionary is not known, then one can resort to the so called \emph{dictionary learning} (DL) to try to learn a sparse representation\cite{tosic2011dictionary}, if one exists. The well-known benchmark algorithm for DL is called $k$-SVD\cite{aharon2006img}, which is a geometry-based algorithm, and can be viewed as a generalization of the clustering algorithms $k$-means and $k$-planes. However, as noted in the original paper, $k$-SVD does not scale well as the size of the dictionary increases. Thus $k$-SVD is often used to construct a dictionary of small image patches of size $8\times 8$, with a few hundreds of atoms.

DL can also be formulated as a matrix factorization problem
\begin{equation}\label{prob:DL}
\begin{aligned}
\minimize_{\D,\Sb}~~& \half\| \Y - \D\Sb \|_F^2 + r(\Sb) \\
\st~~ & \D \in {\cal D},
\end{aligned}
\end{equation}
where $r(\cdot)$ is a sparsity inducing regularization, e.g., the cardinality, the $l_1$ norm, or the log penalty; conceptually there is no need for a constraint on $\D$, however, due to the scaling ambiguity inherent in the matrix factorization problem, we need to impose some norm constraint on the scaling of $\D$ to make the problem better defined. For example, we can bound the norm of each atom in the dictionary, $||\d_i|| \leq 1, \forall i=1,...,k$, where $\d_i$ is the $i$-th column of $\D$, and we adopt this constraint here.

Although bounding the norm of the columns of $\D$ works well, it also complicates the update of $\D$---without this constraint, each row of $\D$ is the solution of an independent least-squares problem sharing the same mixing matrix, while the constraint couples the columns of $\D$, making the problem non-separable. Existing algorithms either solve it approximately\cite{razaviyayn2014dictionary} or by sub-optimal methods like cyclic column updates\cite{mairal2010online}. On the other hand, this is not a problem at all for our proposed ADMM sub-routine Alg.~\ref{alg:ADMMls}: the row separability of the cost function and the column separability of the constraints are handled separately by the two primal variable blocks, while our previously discussed Cholesky caching, warm starting, and good choice of $\rho$ ensure that an exact dictionary update can be done very efficiently.

The update of $\Sb$, sometimes called the sparse coding step, is a relatively well-studied problem for which numerous algorithms have been proposed. We mainly focus on the $l_1$ regularized formulation, in which case the sub-problem becomes the well-known LASSO, and in fact a large number of LASSOs sharing the same mixing matrix. Alg.~\ref{alg:ADMMls} can be used by replacing the proximity step with the soft-thresholding operator. Furthermore, if an over-complete dictionary is trained, the least-squares step can also be accelerated by using the matrix inversion lemma:
\[
(\D^T\D + \rho\I)^{-1} =
\rho^{-1}\I - \rho^{-1}\D^T(\rho\I+\D\D^T)^{-1}\D.
\]
Thus, if $m \ll k$, one can cache the Cholesky of $\rho\I+\D\D^T=\L\L^T$ instead, and replace the least-squares step in Alg.~\ref{alg:ADMMls} with
\[
\tilde{\Sb} \leftarrow \rho^{-1}(\B - \D^T\L^{-T}\L^{-1}\D\B),
\]
where $\B=\D^T\Y + \rho(\Sb+\U)$. The use of ADMM for LASSO is also discussed in \cite{afonso2010fast,yang2011alternating,esser2013method}, and \cite{Boyd2011}, and we generally followed the one described in \cite[\S 6]{Boyd2011}. Again, one should notice that compared to a plain LASSO, our LASSO sub-problem in the AO framework comes with a good initialization, therefore only a very small number of ADMM-iterations are required for convergence.

It is interesting to observe that for the particular constraints and regularization used in DL, incorporating non-negativity maintains the simplicity of our proposed algorithm---for both the norm bound constraint and $l_1$ regularization, the proximity operator in Alg.~\ref{alg:ADMMls} with non-negativity constraint simply requires zeroing out the negative values before doing the same operations. In some applications non-negativity can greatly help the identification of the dictionary \cite{hoyer2004non}.

As an illustrative example, we trained a dictionary from the MNIST handwritten digits dataset\footnote{\url{http://www.cs.nyu.edu/~roweis/data.html}}, which is a collection of gray-scale images of handwritten digits of size $28 \times 28$, and for each digit we randomly sampled 1000 images, forming a matrix of size $784 \times 10,000$. Non-negativity constraints are imposed on both the dictionary and the sparse coefficients. For $k=100$, and by setting the $l_1$ penalty parameter $\lambda=0.5$, the trained dictionary after 100 AO-ADMM (outer-)iterations is shown in Fig.~\ref{fig:mnist}. On average approximately 11 atoms are used to represent each image, and the whole model is able to describe approximately 60\% of the energy of the original data, and the entire training time takes about 40 seconds. Most of the atoms in the dictionary remain readable, which shows the good interpretability afforded by the additional non-negativity constraint.

{\color{blue} For comparison, we tried the same data set with the same parameter settings with the popular and well-developed DL package SPAMS\footnote{\url{http://spams-devel.gforge.inria.fr/index.html}}. For fair comparison, we used SPAMS in batch mode with batch size equal to the size of the training data, and run it for 100 iterations (same number of iterations as AO-ADMM). The quality of the SPAMS dictionary is almost the same as that of AO-ADMM, but it takes SPAMS about 3 minutes to run through these 100 iterations, versus 40 seconds for AO-ADMM. The performance does not change much if we remove the non-negativity constraint when using SPAMS, although the resulting dictionary then loses interpretability. Notice that SPAMS is fully developed in C++, whereas our implementation is simply written in MATLAB, which leaves considerable room for speed improvement using a lower-level language compiler.}
\begin{figure}[t!]
\centering
\includegraphics[width=.6\textwidth]{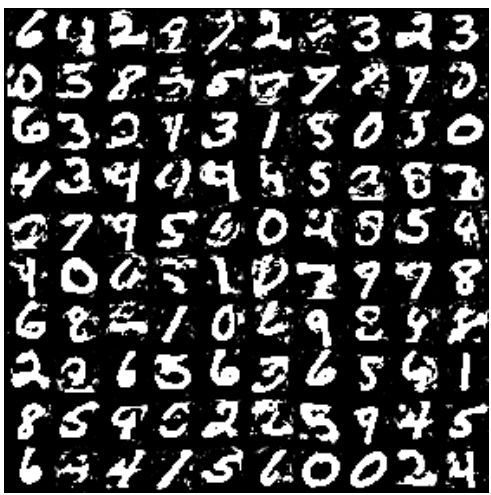}
\vspace{-10pt}
\caption{Trained dictionary from the MNIST handwritten digits dataset.}
\label{fig:mnist}
\end{figure}

\section{Conclusion}
In this paper we proposed a novel AO-ADMM algorithmic framework for matrix and tensor factorization under a variety of constraints and loss functions. The main advantages of the proposed AO-ADMM framework are:
\begin{itemize}
\item {\bf Efficiency.} By carefully adopting AO as the optimization backbone and ADMM for the individual sub-problems, a significant part of the required computations can be effectively cached, leading to a per-iteration complexity similar to the workhorse ALS algorithm for unconstrained factorization. Warm-start that is naturally provided by AO together with judicious regularization and choice of parameters further reduce the number of inner ADMM and outer AO iterations.
\item {\bf Universality.} Thanks to ADMM, which is a special case of the proximal algorithm, non-least-squares terms can be handled efficiently with element-wise complexity using the well-studied proximity operators. This includes almost all non-parametric constraints and regularization penalties commonly imposed on the factors, and even non-least-squares fitting criteria.
\item {\bf Convergence.} AO guarantees monotone decrease of the loss function, which is a nice property for the NP-hard factorization problems considered. Moreover, recent advances on generalizations of the traditional BCD algorithms further guarantee convergence to a stationary point.
\end{itemize}

Case studies on non-negative matrix/tensor factorization, constrained matrix/tensor completion, and dictionary learning, with extensive numerical experiments using real data, corroborate our main claims. We believe that AO-ADMM can serve as a plug-and-play framework that allows easy exploration of different types of constraints and loss functions, as well as different types of matrix and tensor (co-)factorization models.

\bibliographystyle{plain}
\bibliography{ao,admm,nmf,tensor,dl,others}

\begin{thebibliography}{10}
\providecommand{\url}[1]{#1}
\csname url@samestyle\endcsname
\providecommand{\newblock}{\relax}
\providecommand{\bibinfo}[2]{#2}
\providecommand{\BIBentrySTDinterwordspacing}{\spaceskip=0pt\relax}
\providecommand{\BIBentryALTinterwordstretchfactor}{4}
\providecommand{\BIBentryALTinterwordspacing}{\spaceskip=\fontdimen2\font plus
\BIBentryALTinterwordstretchfactor\fontdimen3\font minus
  \fontdimen4\font\relax}
\providecommand{\BIBforeignlanguage}[2]{{%
\expandafter\ifx\csname l@#1\endcsname\relax
\typeout{** WARNING: IEEEtranS.bst: No hyphenation pattern has been}%
\typeout{** loaded for the language `#1'. Using the pattern for}%
\typeout{** the default language instead.}%
\else
\language=\csname l@#1\endcsname
\fi
#2}}
\providecommand{\BIBdecl}{\relax}
\BIBdecl

\bibitem{afonso2010fast}
M.~V. Afonso, J.~M. Bioucas-Dias, and M.~A.~T. Figueiredo, ``Fast image
  recovery using variable splitting and constrained optimization,'' \emph{IEEE
  Trans. on Image Processing}, vol.~19, no.~9, pp. 2345--2356, 2010.

\bibitem{aharon2006img}
M.~Aharon, M.~Elad, and A.~Bruckstein, ``{$k$-SVD}: An algorithm for designing
  overcomplete dictionaries for sparse representation,'' \emph{IEEE Trans. on
  Signal Processing}, vol.~54, no.~11, pp. 4311--4322, 2006.

\bibitem{andersson2000n}
C.~A. Andersson and R.~Bro, ``The {N-way} toolbox for matlab,''
  \emph{Chemometrics and Intelligent Laboratory Systems}, vol.~52, no.~1, pp.
  1--4, 2000.

\bibitem{bader2007efficient}
B.~W. Bader and T.~G. Kolda, ``Efficient {MATLAB} computations with sparse and
  factored tensors,'' \emph{SIAM Journal on Scientific Computing}, vol.~30,
  no.~1, pp. 205--231, 2007.

\bibitem{TTB_Software}
B.~W. Bader, T.~G. Kolda \emph{et~al.}, ``Matlab tensor toolbox version 2.6,''
  \url{http://www.sandia.gov/~tgkolda/TensorToolbox/}, February 2015.

\bibitem{bertsekas1999nonlinear}
D.~P. Bertsekas, \emph{Nonlinear programming}.\hskip 1em plus 0.5em minus
  0.4em\relax Athena Scientific, 1999.

\bibitem{Boyd2011}
S.~P. Boyd, N.~Parikh, E.~Chu, B.~Peleato, and J.~Eckstein, ``Distributed
  optimization and statistical learning via the alternating direction method of
  multipliers,'' \emph{Foundations and Trends{\textregistered} in Machine
  Learning}, vol.~3, no.~1, pp. 1--122, 2011.

\bibitem{bro1997fast}
R.~Bro and S.~De~Jong, ``A fast non-negativity-constrained least squares
  algorithm,'' \emph{Journal of Chemometrics}, vol.~11, no.~5, pp. 393--401,
  1997.

\bibitem{bruckstein2009sparse}
A.~M. Bruckstein, D.~L. Donoho, and M.~Elad, ``From sparse solutions of systems
  of equations to sparse modeling of signals and images,'' \emph{SIAM review},
  vol.~51, no.~1, pp. 34--81, 2009.

\bibitem{cai2013nonnegative}
X.~Cai, Y.~Chen, and D.~Han, ``Nonnegative tensor factorizations using an
  alternating direction method,'' \emph{Frontiers of Mathematics in China},
  vol.~8, no.~1, pp. 3--18, 2013.

\bibitem{candes2011robust}
E.~J. Cand{\`e}s, X.~Li, Y.~Ma, and J.~Wright, ``Robust principal component
  analysis?'' \emph{Journal of the ACM}, vol.~58, no.~3, p.~11, 2011.

\bibitem{candes2009exact}
E.~J. Cand{\`e}s and B.~Recht, ``Exact matrix completion via convex
  optimization,'' \emph{Foundations of Computational Mathematics}, vol.~9,
  no.~6, pp. 717--772, 2009.

\bibitem{candes2008introduction}
E.~J. Cand{\`e}s and M.~B. Wakin, ``An introduction to compressive sampling,''
  \emph{IEEE Signal Processing Magazine}, vol.~25, no.~2, pp. 21--30, 2008.

\bibitem{chen2012maximum}
B.~Chen, S.~He, Z.~Li, and S.~Zhang, ``Maximum block improvement and polynomial
  optimization,'' \emph{SIAM Journal on Optimization}, vol.~22, no.~1, pp.
  87--107, 2012.

\bibitem{chen2009nonnegativity}
D.~Chen and R.~J. Plemmons, ``Nonnegativity constraints in numerical
  analysis,'' in \emph{Symposium on the Birth of Numerical Analysis}, 2009, pp.
  109--140.

\bibitem{choi2014dfacto}
J.~H. Choi and S.~V.~N. Vishwanathan, ``{DFacTo}: Distributed factorization of
  tensors,'' in \emph{Advances in Neural Information Processing Systems}, 2014,
  pp. 1296--1304.

\bibitem{cichocki2009fast}
A.~Cichocki and A.-H. Phan, ``Fast local algorithms for large scale nonnegative
  matrix and tensor factorizations,'' \emph{IEICE Trans. on Fundamentals of
  Electronics, Communications and Computer Sciences}, vol.~92, no.~3, pp.
  708--721, 2009.

\bibitem{duchi2008efficient}
J.~Duchi, S.~Shalev-Shwartz, Y.~Singer, and T.~Chandra, ``Efficient projections
  onto the $l_1$-ball for learning in high dimensions,'' in \emph{Proc. ACM
  ICML}, 2008, pp. 272--279.

\bibitem{eckart1936approximation}
C.~Eckart and G.~Young, ``The approximation of one matrix by another of lower
  rank,'' \emph{Psychometrika}, vol.~1, no.~3, pp. 211--218, 1936.

\bibitem{esser2013method}
E.~Esser, Y.~Lou, and J.~Xin, ``A method for finding structured sparse
  solutions to nonnegative least squares problems with applications,''
  \emph{SIAM Journal on Imaging Sciences}, vol.~6, no.~4, pp. 2010--2046, 2013.

\bibitem{gandy2011tensor}
S.~Gandy, B.~Recht, and I.~Yamada, ``Tensor completion and low-$n$-rank tensor
  recovery via convex optimization,'' \emph{Inverse Problems}, vol.~27, no.~2,
  p. 025010, 2011.

\bibitem{ghadimi2015optimal}
E.~Ghadimi, A.~Teixeira, I.~Shames, and M.~Johansson, ``Optimal parameter
  selection for the alternating direction method of multipliers {(ADMM)}:
  quadratic problems,'' \emph{IEEE Transactions on Automatic Control}, vol.~60,
  no.~3, pp. 644--658, March 2015.

\bibitem{gillis2012accelerated}
N.~Gillis and F.~Glineur, ``Accelerated multiplicative updates and hierarchical
  als algorithms for nonnegative matrix factorization,'' \emph{Neural
  Computation}, vol.~24, no.~4, pp. 1085--1105, 2012.

\bibitem{golub2012matrix}
G.~H. Golub and C.~F. Van~Loan, \emph{Matrix Computations}.\hskip 1em plus
  0.5em minus 0.4em\relax Johns Hopkins University Press, 1996, vol.~3.

\bibitem{Grippo2000}
L.~Grippo and M.~Sciandrone, ``On the convergence of the block nonlinear
  {Gauss}-{Seidel} method under convex constraints,'' \emph{Operations Research
  Letters}, vol.~26, no.~3, pp. 127--136, 2000.

\bibitem{hillar2013most}
C.~J. Hillar and L.-H. Lim, ``Most tensor problems are {NP}-hard,''
  \emph{Journal of the ACM}, vol.~60, no.~6, p.~45, 2013.

\bibitem{hofmann1999probabilistic}
T.~Hofmann, ``Probabilistic latent semantic indexing,'' in \emph{Proc. ACM
  SIGIR Conference}, 1999, pp. 50--57.

\bibitem{hoyer2004non}
P.~O. Hoyer, ``Non-negative matrix factorization with sparseness constraints,''
  \emph{Journal of Machine Learning Research}, vol.~5, pp. 1457--1469, 2004.

\bibitem{huang2014tsp}
K.~Huang, N.~D. Sidiropoulos, and A.~Swami, ``Non-negative matrix factorization
  revisited: Uniqueness and algorithm for symmetric decomposition,'' \emph{IEEE
  Trans. on Signal Processing}, vol.~62, no.~1, pp. 211--224, Jan 2014.

\bibitem{kang2012gigatensor}
U.~Kang, E.~E. Papalexakis, A.~Harpale, and C.~Faloutsos, ``{GigaTensor}:
  scaling tensor analysis up by 100 times-algorithms and discoveries,'' in
  \emph{Proc. ACM SIGKDD}, 2012, pp. 316--324.

\bibitem{kim2008nonnegative}
H.~Kim and H.~Park, ``Nonnegative matrix factorization based on alternating
  nonnegativity constrained least squares and active set method,'' \emph{SIAM
  Journal on Matrix Analysis and Applications}, vol.~30, no.~2, pp. 713--730,
  2008.

\bibitem{kim2011fast}
J.~Kim and H.~Park, ``Fast nonnegative matrix factorization: An active-set-like
  method and comparisons,'' \emph{SIAM Journal on Scientific Computing},
  vol.~33, no.~6, pp. 3261--3281, 2011.

\bibitem{kim2012fast}
------, ``Fast nonnegative tensor factorization with an active-set-like
  method,'' in \emph{High-Performance Scientific Computing}.\hskip 1em plus
  0.5em minus 0.4em\relax Springer, 2012, pp. 311--326.

\bibitem{kolda2009tensor}
T.~G. Kolda and B.~W. Bader, ``Tensor decompositions and applications,''
  \emph{SIAM review}, vol.~51, no.~3, pp. 455--500, 2009.

\bibitem{kruskal1989how}
J.~B. Kruskal, R.~A. Harshman, and M.~E. Lundy, ``How {3-MFA} data can cause
  degenerate {PARAFAC} solutions, among other relationships,'' in
  \emph{Multiway Data Analysis}, 1989, pp. 115--122.

\bibitem{lee1999learning}
D.~D. Lee and H.~S. Seung, ``Learning the parts of objects by non-negative
  matrix factorization,'' \emph{Nature}, vol. 401, no. 6755, pp. 788--791,
  1999.

\bibitem{lee2001algorithms}
------, ``Algorithms for non-negative matrix factorization,'' in \emph{Advances
  in Neural Information Processing Systems (NIPS)}, 2001, pp. 556--562.

\bibitem{liavas2014parallel}
A.~P. Liavas and N.~D. Sidiropoulos, ``Parallel algorithms for constrained
  tensor factorization via the alternating direction method of multipliers,''
  \emph{IEEE Trans. on Signal Processing}, 2014, submitted.

\bibitem{liu2013tensor}
J.~Liu, P.~Musialski, P.~Wonka, and J.~Ye, ``Tensor completion for estimating
  missing values in visual data,'' \emph{IEEE Trans. on Pattern Analysis and
  Machine Intelligence}, vol.~35, no.~1, pp. 208--220, 2013.

\bibitem{mairal2010online}
J.~Mairal, F.~Bach, J.~Ponce, and G.~Sapiro, ``Online learning for matrix
  factorization and sparse coding,'' \emph{Journal of Machine Learning
  Research}, vol.~11, pp. 19--60, 2010.

\bibitem{olshausen1997sparse}
B.~A. Olshausen and D.~J. Field, ``Sparse coding with an overcomplete basis
  set: A strategy employed by {V1}?'' \emph{Vision Research}, vol.~37, no.~23,
  pp. 3311--3325, 1997.

\bibitem{Parikh2014}
N.~Parikh and S.~P. Boyd, ``Proximal algorithms,'' \emph{Foundations and
  Trends{\textregistered} in Optimization}, vol.~1, no.~3, pp. 123--231, 2014.

\bibitem{niranjay}
N.~Ravindran, N.~D. Sidiropoulos, S.~Smith, and G.~Karypis, ``Memory-efficient
  parallel computation of tensor and matrix products for big tensor
  decomposition,'' in \emph{Asilomar Conference on Signals, Systems, and
  Computers}, 2014.

\bibitem{razaviyayn2013unified}
M.~Razaviyayn, M.~Hong, and Z.-Q. Luo, ``A unified convergence analysis of
  block successive minimization methods for nonsmooth optimization,''
  \emph{SIAM Journal on Optimization}, vol.~23, no.~2, pp. 1126--1153, 2013.

\bibitem{razaviyayn2014dictionary}
M.~Razaviyayn, H.-W. Tseng, and Z.-Q. Luo, ``Dictionary learning for sparse
  representation: Complexity and algorithms,'' in \emph{Proc. IEEE ICASSP},
  2014, pp. 5247--5251.

\bibitem{monotone}
E.~Ryu and S.~P. Boyd, ``Primer on monotone operator methods,'' Preprint,
  available at \url{http://web.stanford.edu/~eryu/papers/monotone_notes.pdf}.

\bibitem{sarwar2001item}
B.~Sarwar, G.~Karypis, J.~Konstan, and J.~Riedl, ``Item-based collaborative
  filtering recommendation algorithms,'' in \emph{Proceedings of the 10th
  International Conference on World Wide Web}, 2001, pp. 285--295.

\bibitem{sidiropoulos2000uniqueness}
N.~D. Sidiropoulos and R.~Bro, ``On the uniqueness of multilinear decomposition
  of {N}-way arrays,'' \emph{Journal of chemometrics}, vol.~14, no.~3, pp.
  229--239, 2000.

\bibitem{smilde2005multi}
A.~Smilde, R.~Bro, and P.~Geladi, \emph{Multi-way analysis: applications in the
  chemical sciences}.\hskip 1em plus 0.5em minus 0.4em\relax John Wiley \&
  Sons, 2005.

\bibitem{smith2015splatt}
S.~Smith, N.~Ravindran, N.~D. Sidiropoulos, and G.~Karypis, ``{SPLATT}:
  Efficient and parallel sparse tensor-matrix multiplication,'' in \emph{IEEE
  International Parallel \& Distributed Processing Symposium}, 2015.

\bibitem{tensorlab}
L.~Sorber, M.~Van~Barel, and L.~De~Lathauwer, ``Tensorlab v2.0,''
  \url{http://www.tensorlab.net/}, Jan. 2014.

\bibitem{Stegeman:2014:FLD:2598946.2599268}
\BIBentryALTinterwordspacing
A.~Stegeman, ``Finding the limit of diverging components in three-way
  candecomp/parafac-a demonstration of its practical merits,'' \emph{Comput.
  Stat. Data Anal.}, vol.~75, pp. 203--216, Jul. 2014. [Online]. Available:
  \url{http://dx.doi.org/10.1016/j.csda.2014.02.010}
\BIBentrySTDinterwordspacing

\bibitem{sun2014alternating}
D.~L. Sun and C.~Fevotte, ``Alternating direction method of multipliers for
  non-negative matrix factorization with the beta-divergence,'' in \emph{Proc.
  IEEE ICASSP}, 2014, pp. 6201--6205.

\bibitem{tomasi2006comparison}
G.~Tomasi and R.~Bro, ``A comparison of algorithms for fitting the {PARAFAC}
  model,'' \emph{Computational Statistics \& Data Analysis}, vol.~50, no.~7,
  pp. 1700--1734, 2006.

\bibitem{tosic2011dictionary}
I.~Tosic and P.~Frossard, ``Dictionary learning,'' \emph{IEEE Signal Processing
  Magazine}, vol.~28, no.~2, pp. 27--38, 2011.

\bibitem{tseng2001convergence}
P.~Tseng, ``Convergence of a block coordinate descent method for
  nondifferentiable minimization,'' \emph{Journal of Optimization Theory and
  Applications}, vol. 109, no.~3, pp. 475--494, 2001.

\bibitem{van2004fast}
M.~H. Van~Benthem and M.~R. Keenan, ``Fast algorithm for the solution of
  large-scale non-negativity-constrained least squares problems,''
  \emph{Journal of Chemometrics}, vol.~18, no.~10, pp. 441--450, 2004.

\bibitem{vavasis2009complexity}
S.~A. Vavasis, ``On the complexity of nonnegative matrix factorization,''
  \emph{SIAM Journal on Optimization}, vol.~20, no.~3, pp. 1364--1377, 2009.

\bibitem{xu2013block}
Y.~Xu and W.~Yin, ``A block coordinate descent method for regularized
  multiconvex optimization with applications to nonnegative tensor
  factorization and completion,'' \emph{SIAM Journal on Imaging Sciences},
  vol.~6, no.~3, pp. 1758--1789, 2013.

\bibitem{xu2012alternating}
Y.~Xu, W.~Yin, Z.~Wen, and Y.~Zhang, ``An alternating direction algorithm for
  matrix completion with nonnegative factors,'' \emph{Frontiers of Mathematics
  in China}, vol.~7, no.~2, pp. 365--384, 2012.

\bibitem{yang2011alternating}
J.~Yang and Y.~Zhang, ``Alternating direction algorithms for $l_1$-problems in
  compressive sensing,'' \emph{SIAM Journal on Scientific Computing}, vol.~33,
  no.~1, pp. 250--278, 2011.

\end{thebibliography}
\end{document}